\documentclass[11pt]{article}
\usepackage{graphicx,amsmath,amsfonts,amssymb,bm,hyperref,url,breakurl,epsfig,epsf,color,fullpage,MnSymbol,mathbbol, fmtcount, algorithm, algorithmic, semtrans
} 

\usepackage{titlesec}

\usepackage{tikz}
\usepackage{pgfplots}

\usetikzlibrary{pgfplots.groupplots}

\titleformat{\paragraph}
{\normalfont\normalsize\bfseries}{\theparagraph}{1em}{}
\titlespacing*{\paragraph}
{0pt}{3.25ex plus 1ex minus .2ex}{1.5ex plus .2ex}

\usepackage{caption,subcaption}

\usepackage[bottom,hang,flushmargin]{footmisc} 

\usepackage{hyperref}
\definecolor{darkred}{RGB}{150,0,0}
\definecolor{darkgreen}{RGB}{0,150,0}
\definecolor{darkblue}{RGB}{0,0,200}
\hypersetup{colorlinks=true, linkcolor=darkred, citecolor=darkgreen, urlcolor=black}

\newtheorem{theorem}{Theorem}[section]
\newtheorem{lemma}[theorem]{Lemma}

\newtheorem{definition}[theorem]{Definition}

\newtheorem{remark}[theorem]{Remark}

\newcommand{\dist}{\mathcal{D}}
\newcommand{\reals}{\mathbb{R}}
\newcommand{\X}{\mathcal{X}}
\newcommand{\Gt}{\vct{g}_t}
\newcommand{\Dphi}{D_{\Phi}}

\newcommand{\z}{{\mtx{z}}}

\newcommand{\y}{\vct{y}}
\newcommand{\C}{\mathcal{K}}

\newcommand{\x}{\vct{x}}

\newcommand{\twonorm}[1]{\left\|#1\right\|_{\ell_2}}

\newcommand{\R}{\mathbb{R}}
\newcommand{\Z}{\mathbb{Z}}

\newcommand{\E}{\operatorname{\mathbb{E}}}

\newcommand{\e}{\mathrm{e}}

\newcommand{\vct}[1]{\bm{#1}}
\newcommand{\mtx}[1]{\bm{#1}}

\definecolor{ejc}{RGB}{0,0,255}

\numberwithin{equation}{section} 

\def \endprf{\hfill {\vrule height6pt width6pt depth0pt}\medskip}
\newenvironment{proof}{\noindent {\bf Proof} }{\endprf\par}

\newcommand*\samethanks[1][\value{footnote}]{\footnotemark[#1]}

\title{Gradient Methods for Submodular Maximization}
\author{Hamed Hassani\thanks{H. Hassani and M. Soltanolkotabi contributed equally.}~\thanks{Department of Electrical and Systems Engineering, University of Pennsylvania, Philadelphia, PA}\quad Mahdi Soltanolkotabi\samethanks[1]~\thanks{Ming Hsieh Department of Electrical Engineering, University of Southern California, Los Angeles, CA} \quad  Amin Karbasi\thanks{Departments of Electrical Engineering and Computer Science, Yale University, New Haven, CT}}

\date{May, 2017}
\begin{document}
\maketitle
\begin{abstract}
In this paper, we study the problem of maximizing continuous submodular functions that naturally arise in many learning applications such as those involving utility functions in active learning and sensing, matrix approximations and network inference. Despite the apparent lack of convexity in such functions, we prove that stochastic projected gradient methods can provide strong approximation guarantees for maximizing continuous submodular functions with convex constraints. More specifically, we prove that for monotone continuous DR-submodular functions, all fixed points of projected gradient ascent provide a factor $1/2$ approximation to the global maxima. We also study stochastic gradient and mirror methods and show that after $\mathcal{O}(1/\epsilon^2)$ iterations these methods reach solutions which achieve in expectation objective values exceeding $(\frac{\text{OPT}}{2}-\epsilon)$. 
An immediate application of our results is to maximize submodular functions that are defined stochastically, i.e. the submodular function is defined as an expectation over a family of submodular functions with an unknown distribution. We will show how stochastic gradient methods are naturally well-suited for this setting, leading to a factor $1/2$ approximation when the function is monotone. In particular, it allows us to approximately maximize discrete, monotone submodular optimization problems via projected gradient descent on a continuous relaxation, directly connecting the discrete and continuous domains. Finally, experiments on real data demonstrate that our projected gradient methods consistently achieve the best utility compared to other continuous baselines while remaining competitive in terms of computational effort.
   

%
%
%
\end{abstract}

\section{Introduction}\label{sec:intro}
Submodular set functions exhibit a natural diminishing returns property, resembling concave functions in continuous domains. At the same time, they can be minimized exactly in polynomial time (while can only be   maximized approximately), which makes them similar to convex functions. They have found  numerous applications in machine learning, including  viral marketing \cite{kempe03}, dictionary learning \cite{das2011submodular} network monitoring \cite{leskovec07, gomez10}, sensor placement  \cite{guestrin2005near}, product recommendation \cite{el2009turning, mirzasoleimanfast},  document and corpus summarization \cite{lin2011class} data summarization \cite{mirzasoleiman2013distributed}, crowd teaching \cite{singla2014near, KimKK16}, and probabilistic models \cite{djolonga2014map, iyer2015submodular}. However, submodularity  is in general a property that goes  beyond 
set functions and can be defined for continuous functions. In this paper, we consider the following \textit{stochastic} continuous submodular  optimization problem:
\begin{equation}\label{eq:mainproblem}
\max_{x\in\C}F(x) \doteq \E_{\theta\sim \dist} [F_\theta(x)],
\end{equation}
where $\C$ is a bounded convex body, $\dist$ is generally an \textit{unknown} distribution, and $F_\theta$'s are are continuous submodular functions for every $\vct{\theta}\in\mathcal{D}$. Such a setting has recently been introduced in \cite{karimi17}. We also denote the optimum value as $\text{OPT}\triangleq \max_{\x \in \mathcal{K}} F(\x)$.  We note that the function $F(x)$ is itself also continuous submodular as a non-negative combination of  submodular functions are still submodular \cite{bach2015submodular}. The formulation covers popular instances of submodular optimization. For instance, when $\dist$ puts all the probability mass on a single function, \eqref{eq:mainproblem} reduces to \textit{deterministic} continuous submodular  optimization.  Another  common objective is the \textit{finite-sum} continuous submodular  optimization  where $\dist$ is uniformly distributed over $m$ instances, i.e., $F(x) \doteq \frac{1}{m}\sum_{\theta=1}^m F_\theta(x)$. 

A natural approach to solving problems of the form \eqref{eq:mainproblem} is to use projected stochastic methods. 
As we shall see in Section~\ref{sec:experiments}, these local search heuristics are surprisingly effective. However, the reasons for this empirical success is completely unclear. The main challenge is that maximizing $F$ corresponds to a nonconvex optimization problem (as the function $F$ is not concave), and a priori it is not clear why gradient methods should yield a reliable solution. This leads us to the main challenge of this paper
\begin{quote}
Do projected gradient  methods lead to \textit{provably good solutions} for  continuous submodular maximization with general convex constraints?
\end{quote}
We answer the above question in  the affirmative, proving that projected gradient methods produce a competitive solution with respect to the optimum.   
%
%
More specifically, given a general bounded convex body $\C$  and a continuous  function $F$ that is monotone, smooth, and  (weakly) DR-submodular we show that 
\begin{itemize}
\item All stationary points of a DR-submodular function $F$ over $\C$ provide a $1/2$ approximation to the global maximum. Thus, projected gradient methods with sufficiently small step sizes (a.k.a.~gradient flows) always lead to a solutions with $1/2$ approximation guarantees.
\item  Projected gradient ascent after $O\left(\frac{L_2}{\epsilon}\right)$ iterations produces a solution with objective value larger than $(\text{OPT}/2-\epsilon)$. When calculating the gradient is difficult but an unbiased estimate can be easily obtained,  the stochastic projected gradient ascent in $O\left(\frac{L_2}{\epsilon}+\frac{\sigma^2}{\epsilon^2}\right)$ iterations finds a solution with objective value exceeding $(\text{OPT}/2-\epsilon)$. Here, 
$L_2$ is the smoothness of the continuous submodular function measured in the $\ell_2$-norm, $\sigma^2$ is the variance of the stochastic gradient with respect to the true gradient and OPT is the function value at the global optimum.
%
%
\item Projected mirror ascent after  $O\left(\frac{L_*}{\epsilon}\right)$ iterations produces a solution with objective value larger than$(\text{OPT}/2-\epsilon)$. Similarly, the stochastic projected mirror ascent in $O\left(\frac{L_*}{\epsilon}+\frac{\sigma^2}{\epsilon^2}\right)$ iterations finds a solution with objective value exceeding $(\text{OPT}/2-\epsilon)$. Crucially,  $L_*$ indicates the smoothness of the continuous submodular function measured in any norm (e.g., $\ell_1$) that can be substantially smaller than the $\ell_2$ norm. 
\item More generally, for weakly continuous DR-submodular functions with parameter $\gamma$  (define in \eqref{eq:weaksub}) we prove the above results with $\gamma^2/(1+\gamma^2)$ approximation guarantee. 

%
\end{itemize}
Our result have some important implications. First, they show that projected gradient methods are an efficient way of maximizing the multilinear extension of (weakly) submodular set functions for any submodularity ratio $\gamma$ (note that $\gamma=1$ corresponds to submodular functions) \cite{das2011submodular}. Second, in contrast to conditional gradient methods for submodular maximization that should always start from the origin \cite{calinescu11maximizing, bian2016guaranteed}, projected gradient methods can start from any initial point in the constraint set $\C$ and still produce a competitive solution. Third, such conditional gradient methods, when applied to the stochastic setting (with a fixed batch size), perform poorly and can produce arbitrarily bad solutions when applied to continuous submodular functions (see Appendix \ref{bad} for an example and further discussion on why conditional gradient methods do not easily admit stochastic variants).  In contrast, stochastic projected gradient methods are stable  by design and  provide a  solution with $1/2$ guarantee in expectation. Finally, our work provides a unifying approach for solving the \textit{stochastic submodular maximization problem} \cite{karimi17}
\begin{equation}\label{eq:stochsub}
f(S) \doteq \E_{\theta\sim\dist}[f_\theta(S)],
\end{equation}
where the functions $f_\theta:2^V\rightarrow\reals_+$ are submodular set functions defined over the ground set $V$.   Such objective functions naturally arise in many data summarization applications \cite{serban17} and have been recently introduced and studied in \cite{karimi17}. 
Since $\dist$ in unknown, problem~\eqref{eq:stochsub} cannot be directly solved. Instead, \cite{karimi17} showed that in the case of coverage functions, it is possible to efficiently maximize $f$ by lifting the problem to the continuous domain and  using stochastic gradient methods on a continuous relaxation to reach a solution that is within a factor $(1-1/e)$ of the optimum. In contrast, our work provides a general recipe with $1/2$ approximation guarantee for problem~\eqref{eq:stochsub} in which $f_\theta$'s can be any monotone submodular function. 

\section{Continuous submodular maximization}\label{sec:problem}
A set function $f:2^V\rightarrow \reals_+$, defined on the ground set $V$,  is called submodular if for all subsets $A,B\subseteq V$, we have $$f(A)+f(B)\geq f(A\cap B) + f(A\cup B).$$ Even though submodularity is mostly considered on discrete domains, the notion can
be naturally extended to arbitrary lattices \cite{fujishige91}. To this aim, let us consider a subset of $\reals_+^n$ of the form $\X = \prod_{i=1}^n \X_i$ where each $\X_i$ is a compact subset of $\reals_+$. A function $F:\X\rightarrow \reals_+$ is  \textit{submodular} if for all $(\vct{x},\y)\in \X \times \X$, we have
\begin{equation}\label{eq:sub1}
F(\x)+ F(\y) \geq F(\x\vee \y) + F(x\wedge \y),
\end{equation}
where $\x\vee \y \doteq\max(\x,\y)$ (component-wise) and $\x\wedge \y \doteq \min (\x,\y)$ (component-wise). A submodular function is monotone if for any $\x,\y\in\X$ such that $\x\leq \y$, we have $F(\x)\leq F(\y)$ (here, by $\x \leq \y$ we mean that every element of $\x$ is less than that of $\y$). Like set functions, we can define submodularity in an equivalent way, reminiscent of diminishing returns,  as follows \cite{bach2015submodular}: the function $F$ is submodular if for  any $\x \in \X$ and two distinct basis vectors  $\e_i, \e_j \in \reals^n$ and two non-negative real numbers $z_i, z_j\in \reals_+$, such that $\x_i+z_i\in\X_i$ and $\x_j+z_j\in\X_j$, then
\begin{equation}\label{eq:sub2}
F(\x+z_i\e_i) + F(\x+z_j\e_j)\geq F(\x) + F(\x+z_i\e_i+z_j\e_j). 
\end{equation}
Clearly, the above definition includes submodularity over a set (by restricting $\X_i$'s to  $\{0,1\}$) or over  an integer lattice (by restricting $\X_i$'s to $\Z_+$) as special cases. However, in the remaining of this paper we consider \textit{continuous} submodular functions defined on product of sub-intervals of $\reals_+$.  When  twice differentiable, $F$ is submodular if and only if all cross-second-derivatives are non-positive \cite{bach2015submodular}, i.e., 
\begin{equation}\label{eq:twice}
\forall i\neq j, \forall \x\in \X, ~~ \frac{\partial^2 F(\vct{x})}{\partial x_i \partial x_j} \leq 0.
\end{equation}
The above expression makes it clear that continuous submodular functions are not convex nor concave in general as concavity (convexity) implies that $\nabla^2 F\preceq 0$ (resp.$\bigtriangledown^2 F \succeq 0$). Indeed, we can have functions that are both submodular and convex/concave. For instance, for a concave function $g$ and non-negative weights $\lambda_i\geq 0$, the function $F(\x) = g(\sum_{i=1}^n \lambda_i x_i)$ is  submodular and concave. Trivially, affine functions are submodular, concave, and convex. A proper subclass of submodular functions are called \textit{DR-submodular} \cite{bian2016guaranteed, soma2015generalization} if for all $\x,\y\in\X$ such that $\x\leq \y$  and any standard basis vector $\e_i\in\reals^n$ and a non-negative number $z\in\reals_+$ such that $z\e_i+\x\in\X$ and $z\e_i+\y\in\X$, then,
\begin{equation}\label{eq:drsub}
F(z\e_i+\x)-F(\x)\geq F(z\e_i+\y)-F(\y). 
\end{equation} 
One can easily verify that for a differentiable DR-submodular functions the gradient is an antitone  mapping, i.e., for all $\x,\y \in \X$ such that $\x\leq \y$ we have  $\nabla F(\x) \geq \nabla F(\y)$ \cite{bian2016guaranteed}.   When twice differentiable, DR-submodularity is equivalent to 
\begin{equation}\label{eq:twicedr}
\forall i~ \&~ j, \forall \x\in \X, ~~ \frac{\partial^2 F(\vct{x})}{\partial x_i \partial x_j} \leq 0.
\end{equation}
The above twice differentiable functions are sometimes called \textit{smooth} submodular functions in the literature \cite{chekuri2015multiplicative}. However, in this paper, we say a differentiable submodular function $F$ is $L$-\textit{smooth} w.r.t a norm $\Vert\cdot\Vert$ (and its dual norm $\Vert\cdot\Vert_*$)  if for all $\x,\y\in\X$ we have 
$$\Vert\nabla F(\x)- \nabla F(\x)\Vert_*\leq L \Vert \x-\y\Vert.$$
Here, $\|\cdot\|_{*}$ is the \textit{dual norm} of $\|\cdot\|$ defined as $\|\vct{g}\|_{*}=\sup_{\vct{x}\in\R^n:\text{ }\|\vct{x}\|\le 1} \vct{g}^T\vct{x}.$
 When the function is smooth w.r.t the $\ell_2$-norm we use $L_2$ (note that the $\ell_2$ norm is self-dual).  We say that a function is \textit{weakly DR-submodular} with  parameter $\gamma$ if
\begin{equation}\label{eq:weaksub}
\gamma = \inf_{\substack{x,y\in\X\\\x\leq \y}}\inf_{i\in[n]}\frac{[\nabla F(\x)]_i}{[\nabla F(\y)]_i}.
\end{equation}
Clearly, for a differentiable DR-submodular function we have $\gamma=1$. An important example of a DR-submodular function is the multilinear extension \cite{calinescu11maximizing} $F: [0,1]^n\rightarrow \reals$ of a discrete submodular function $f$, namely,  $$F(\x) = \sum_{S\subseteq V} \prod_{i\in S} x_i \prod_{j\not\in S} (1-x_j) f(S).$$   
We note that for set functions, DR-submodularity (i.e., Eq.~\ref{eq:drsub}) and submodularity (i.e., Eq.~\ref{eq:sub1}) are equivalent. However, this is not true for the general submodular functions defined on integer lattices or product of sub-intervals \cite{bian2016guaranteed, soma2015generalization}.

The focus of this paper is on  continuous submodular maximization defined in Problem~\eqref{eq:mainproblem}. More specifically, we assume that $\C\subset \X$ is a a general bounded convex set (not necessarily down-closed as considered in \cite{bian2016guaranteed}) with diameter $R$. Moreover, we consider $F_\theta$'s to be monotone  (weakly) DR-submodular functions with parameter $\gamma$. 
%

\section{Background and related work}\label{sec:related}
Submodular set functions \cite{edmonds1971matroids, fujishige91} originated in combinatorial optimization
and operations research, but they have recently attracted significant interest in machine learning. Even though they are  usually considered  over discrete domains, their optimization is inherently related to continuous optimization methods. In particular, Lovasz \cite{lovasz1983submodular} showed that Lovasz extension is convex if and only if the corresponding set function is submodular. Moreover, minimizing a submodular set-function  is equivalent to minimizing the Lovasz extension.\footnote{The idea of using stochastic methods for submodular minimization has recently been used in \cite{yantatlee}.} This idea has been recently extended to minimization of strict continuous submodular functions (i.e., cross-order derivatives in \eqref{eq:twice} are strictly negative) \cite{bach2015submodular}. Similarly, approximate submodular maximization is linked to a different continuous extension known as multilinear extension \cite{chekuri11submodular}. Multilinear extension (which is an example of DR-submodular functions studied in this paper) is not concave nor convex in general. However, a variant of conditional gradient method, called \textit{continuous greedy},  can be used to approximately maximize them. Recently, Chekuri et al \cite{chekuri2015multiplicative} proposed an interesting multiplicative weight update algorithm that achieves $(1-1/e-\epsilon)$ approximation guarantee after $\tilde{O}(n^2/\epsilon^2)$ steps for twice differentiable  monotone DR-submodular functions (they are also called smooth submodular functions) subject to a polytope constraint.  Similarly, Bian et al \cite{bian2016guaranteed} proved that  a conditional gradient method, similar to the continuous greedy algorithm, achieves  $(1-1/e-\epsilon)$ approximation guarantee after $O(L_2/\epsilon)$ iterations   for  maximizing a monotone DR-submodular functions subject to  special convex constraints called \textit{down-closed} convex bodies. 
 A few remarks are in order. First, the proposed conditional gradient methods cannot handle the general stochastic setting we consider in Problem~\eqref{eq:mainproblem} (in fact, projection is the key). Second, there is no near-optimality guarantee if  conditional gradient methods do not start from the origin. More precisely, for the continuous greedy algorithm it is necessary to start from the $\vct{0}$ vector (to be able to remain in the convex constraint set at each iteration). Furthermore, the $\vct{0}$ vector must be a feasible point of the constraint set. Otherwise, the iterates of the algorithm may fall out of the convex constraint set leading to an infeasible final solution. Third, due to the starting point requirement, they can only handle special convex constraints, called down-closed. And finally, the dependency on $L_2$ is very subomptimal as it can be as large as the dimension $n$ (e.g., for the  multilinear extensions of some submodular set functions, see Appendix~\ref{example-smooth}). Our work resolves all of these issues by showing that projected gradient methods can also approximately maximize monotone DR-submodular functions subject to general convex constraints, albeit, with a lower $1/2$ approximation guarantee. 

Generalization of submodular set functions has lately received a lot of attention. For instance, a line of recent work considered DR-submodular function maximization over an integer lattice \cite{soma2014optimal, gottschalk2015submodular, soma2015generalization}. Interestingly, Ene and Nguyen \cite{ene2016reduction} provided an efficient reduction from an integer-lattice DR-submodular to a submodular set function, thus suggesting a simple way to solve integer-lattice DR-submodular maximization. Note that such reductions cannot be applied to the optimization problem~\eqref{eq:mainproblem} as  expressing general  convex body constraints may require solving a continuous optimization problem. 

%
%

\section{Algorithms and main results}\label{sec:results}
In this section we discuss our algorithms together with the corresponding theoretical guarantees. In what follows, we assume that $F$ is a weakly DR-submodular function with parameter $\gamma$.   
\subsection{Characterizing the quality of stationary points} \label{sec:stationary}
We begin with the definition of a stationary point. 
\begin{definition} \label{def:local_maxima}
A vector $\x \in \mathcal{K}$ is called a \textit{stationary point} of a function $F:\mathcal{X} \rightarrow \mathbb{R}_+$ over the set $\mathcal{K}\subset \mathcal{X}$ if $\max_{\vct{y} \in \mathcal{K}} \langle  \nabla F(x), \vct{y} - \vct{x} \rangle \leq 0$. 
\end{definition}
Stationary points are of interest because they characterize the fixed points of the Gradient Ascent (GA) method. Furthermore, (projected) gradient ascent with a sufficiently small step size is known to converge to a stationary point for smooth functions \cite{nesterov2013introductory}. To gain some intuition regarding this connection, let us consider the GA procedure. Roughly speaking, at any iteration $t$ of the GA procedure, the value of $F$ increases (to the first order) by $\langle \nabla F(\x_t), \x_{t+1} - \x_t \rangle$. Hence, the progress at time $t$ is at most $\max_{\y \in \C} \langle  \nabla F(\x_t), \y - \x_t \rangle$.  If at any time $t$ we have  $\max_{\y \in \C} \langle  \nabla F(\x_t), \y - \x_t \rangle \leq 0$, then the GA procedure will not make any progress and it will be stuck once it falls into a stationary point. 

The next natural question is how small can the value of $F$ be at a stationary point compared to the global maximum? The following lemma relates the value of $F$ at a stationary point to OPT.  
\begin{theorem} \label{local_opt}
Let $F: \X \to \mathbb{R}_+$ be monotone and weakly DR-submodular with parameter $\gamma$ and assume $\mathcal{K} \subseteq \X$ is a convex set. Then,
\begin{itemize}
\item [(i)] If $\x$ is a stationary point of $F$ in $\mathcal{K}$, then $F(\x) \geq \frac{\gamma^2}{1 + \gamma^2} \rm{OPT}$. 
\item[(ii)] Furthermore, if $F$ is $L$-smooth, gradient ascent with a step size smaller than $1/L$ will converge to a stationary point.
\end{itemize}
\end{theorem}
The theorem above guarantees that all fixed points of the GA method yield a solution whose function value is at least $\frac{\gamma^2}{1+\gamma^2}$OPT. Thus, all fixed point of GA provide a factor $\frac{\gamma^2}{1+\gamma^2}$ approximation ratio. The particular case of $\gamma = 1$, i.e.,~when $F$ is DR-submodular, asserts that at any stationary point $F$ is at least $\text{OPT}/2$. This lower bound is in fact tight. In Appendix \ref{exmponehalf} we provide a simple instance of a differentiable DR-Submodular function that attains $\text{OPT}/2$ at a stationary point that is also a \textit{local maximum}. 
\subsection{(Stochastic) gradient methods}
We now discuss our first algorithmic approach. For simplicity we focus our exposition on the DR submodular case, i.e.,~$\gamma=1$, and discuss how this extends to the more general case in the proofs (Section \ref{weaksub}). A simple approach to maximizing DR submodular functions is to use the (projected) Gradient Ascent (GA) method. Starting from an initial estimate $\vct{x}_1\in\mathcal{K}$ obeying the constraints, GA iteratively applies the following update
\begin{align}
\label{PGDupdates}
\vct{x}_{t+1}=\mathcal{P}_{\mathcal{K}}\left(\vct{x}_t+\mu_t\nabla F(\vct{x}_t)\right).
\end{align}
Here, $\mu_t$ is the learning rate and $\mathcal{P}_{\mathcal{K}}(\vct{v})$ denotes the Euclidean projection of $\vct{v}$ onto the set $\mathcal{K}$. However, in many problems of practical interest we do not have direct access to the gradient of $F$. In these cases it is natural to use a stochastic estimate of the gradient in lieu of the actual gradient. This leads to the Stochastic Gradient Method (SGM). Starting from an initial estimate $\vct{x}_0\in\mathcal{K}$ obeying the constraints, SGM iteratively applies the following updates
\begin{align}
\label{SGAupdates}
\vct{x}_{t+1}=\mathcal{P}_{\mathcal{K}}\left(\vct{x}_t+\mu_t\vct{g}_t\right).
\end{align}
Specifically, at every iteration $t$, the current iterate $\vct{x}_t$ is updated by adding $\mu_t \vct{g}_t$, where $\vct{g}_t$ is an unbiased estimate of the gradient $\nabla F(\vct{x}_t)$ and $\mu_t$ is the learning rate. The result is then projected onto the set $\mathcal{K}$. We note that when $\vct{g}_t = \nabla F(\vct{x}_t)$, i.e., when there is no randomness in the updates, then the SGM updates \eqref{SGAupdates} reduce to the GA updates \eqref{PGDupdates}. We detail the SGM method in Algorithm~\ref{alg:sga}.
\begin{algorithm}[t]
    \caption{(Stochastic) Gradient Method for Maximizing $F(x)$ over a convex set $\C$}
    \label{alg:sga}
    \begin{algorithmic}
    \STATE {\bfseries Parameters:} Integer $T >0$ and scalars $\eta_t >0$, $t \in [T]$
   \STATE {\bfseries Initialize:}  $\x_1  \in \C$
   \FOR{$t=1$ \textbf{to} $T$}
   \STATE $\y_{t+1} \leftarrow \x_t + \eta_t \vct{g}_t$, 
   \STATE \quad with $\vct{g}_t$ s.t. $\mathbb{E}[\vct{g}_t | \x_t] = \nabla F(\x_t)$ where $\vct{g}_t$ is a random vector s.t. $\mathbb{E}[\vct{g}_t | \x_t] = \nabla F(\x_t)$
   \STATE $\x_{t+1} = \arg\min_{\x \in \C} || \x - \y_{t+1}||_2$
   \ENDFOR 
   \STATE{Pick $\tau$ uniformly at random from $\{1,2,\ldots,T\}$.}
    \STATE{\bfseries Output} $\x_\tau$
    \end{algorithmic}
  \end{algorithm}

As we shall see in our experiments detained in Section~\ref{sec:experiments},  the SGM method is surprisingly effective for maximizing monotone DR-submodular functions. However, the reasons for this empirical success was previously unclear. The main challenge is that maximizing $F$ corresponds to a nonconvex optimization problem (as the function $F$ is not concave), and a priori it is not clear why gradient methods should yield a competitive ratio. Thus, studying gradient methods for such nonconvex problems poses new challenges: 
\begin{quote}
Do (stochastic) gradient methods converge to a stationary point? 
\end{quote}

The next theorem addresses some of these challenges. To be able to state this theorem  let us recall the standard definition of smoothness.  
 We say that a continuously differentiable function $F$ is $L$-smooth (in Euclidean norm) if the gradient $\nabla F$ is $L$-Lipschitz, that is $\twonorm{\nabla F(\vct{x})-\nabla F(\vct{y})}\le L \twonorm{\vct{x}-\vct{y}}.$ We also defined the diameter (in Euclidean norm) as $R^2={\sup}_{\vct{x},\vct{y}\in\mathcal{K}} \frac{1}{2}\twonorm{\vct{x}-\vct{y}}^2$.
%
We now have all the elements in place to state our first theorem.
\begin{theorem}[Stochastic Gradient Method]\label{thm:sgd} 
Let us assume that $F$ is $L$-smooth w.r.t. the Euclidean norm $\twonorm{\cdot}$, monotone and DR-submodular. 
Furthermore, assume that we have access to a stochastic oracle $\Gt$ obeying $$\E[\Gt]=\nabla F(\x_t)\quad\text{and}\quad \E\big[\twonorm{\Gt-\nabla F(\x_t)}^2\big]\le \sigma^2.$$
We run stochastic gradient updates of the form \eqref{SGAupdates} with $\mu_t=\frac{1}{L+\frac{\sigma}{R}\sqrt{t}}$. Let $\tau$ be a random variable taking values in $\{1,2,\ldots,T\}$ with equal probability. Then,
\begin{align}
\label{optguarantee}
\E[F(\x_\tau)] \ge  \frac{\rm{OPT}}{2} -\left(\frac{R^2L+\text{OPT}}{2T}+\frac{R\sigma}{\sqrt{T}}\right).
\end{align}
\end{theorem} 
\begin{remark}
We would like to note that if we pick $\tau$ to be a random variable taking values in $\{2,\ldots,T-1\}$ with probability $\frac{1}{(T-1)}$ and $1$ and $T$ each with probability $\frac{1}{2(T-1)}$ then
\begin{align*}
\E[F(\x_\tau)] \ge  \frac{\rm{OPT}}{2} -\left(\frac{R^2L}{2T}+\frac{R\sigma}{\sqrt{T}}\right).
\end{align*}
\end{remark}
The above results roughly state that $T=\mathcal{O}\left(\frac{R^2L}{\epsilon}+\frac{R^2\sigma^2}{\epsilon^2}\right)$ iterations of the stochastic gradient method from any initial point, yields a solution whose objective value is at least $\frac{\text{OPT}}{2} - \epsilon$. Stated differently, $T=\mathcal{O}\left(\frac{R^2L}{\epsilon}+\frac{R^2\sigma^2}{\epsilon^2}\right)$ iterations of the stochastic gradient method provides in expectation a value that exceeds $\frac{\text{OPT}}{2} - \epsilon$ approximation ratio for DR-submodular maximization. As explained in Section \ref{sec:stationary}, it is not possible to go beyond the factor $1/2$ approximation ratio using gradient ascent from an arbitrary initialization. 





An important aspect of the above result is that it only requires an unbiased estimate of the gradient. This flexibility is crucial for many DR-submodular maximization problems (see, \eqref{eq:mainproblem}) as in many cases calculating the function $F$ and its derivative is not feasible. However, it is possible to provide a good un-biased estimator for these quantities. 

We would like to point out that our results are similar in nature to known results about stochastic methods for convex optimization. Indeed, this result interpolates between the $\frac{1}{\sqrt{T}}$ for stochastic smooth optimization, and the $1/T$ for deterministic smooth optimization. The special case of $\sigma = 0$ which corresponds to Gradient Ascent deserves particular attention. In this case, and under the assumptions of Theorem~\ref{thm:sgd}, it is possible to show that $F(\x_T)\ge \frac{\rm{OPT}}{2}- \frac{R^2L}{T}$, without the need for a randomized choice of $\tau \in [T]$.

Finally, we would like to note that while the first term in \eqref{optguarantee} decreases as $1/T$, the pre-factor $L$ could be rather large in many applications. For instance, this quantity may depend on the dimension of the input $n$ (see Section~\ref{example-smooth} in the Appendix). Thus, the number of iterations for reaching a desirable accuracy may be very large. Such a large computational load causes (stochastic) gradient methods infeasible in some application domains. We will overcome this deficiency in the next section by using stochastic mirror methods.

\subsection{Stochastic mirror method}  
In the previous section we saw that when the function $F$ and the constraint set $\mathcal{K}$ are well-behaved in the Euclidean norm (e.g.,~$L$ is a constant in the $\ell_2$ norm) then the total number of iterations to reach a certain accuracy is dimension-free and independent of the ambient dimension $n$. However, in many cases of interest, including some that arise from the multilinear extension of discrete submodular functions, the smoothness parameter scales with the ambient dimension and thus the number of iterations will be dimension dependent. This is particularly problematic for large-scale applications where the ambient dimension is very large. However, smoothness when measured in a different norm may still be dimension independent. Indeed, multilinear relaxation of discrete submodular functions have smoothness parameter in $\ell_1$ norm that is bounded by their maximum singleton value (see Section~\ref{example-smooth} in the appendices). In this section we discuss our results for Mirror methods which are designed to adapt to smoothness in general norms. To explain the mirror descent method we need a few definitions. First, recall the definition of smoothness to arbitrary norms:
 a continuously differentiable function $F$ is $L$-smooth with respect to a norm $\|\cdot\|$ if the gradient $\nabla F$ obeys
$\|\nabla F(\vct{x})-\nabla F(\vct{y})\|_{*}\le L \|\vct{x}-\vct{y}\|.$
Here, $\|\cdot\|_{*}$ is the dual norm of $\|\cdot\|$ defined as $\|\vct{g}\|_{*}=\sup_{\vct{x}\in\R^n:\text{ }\|\vct{x}\|\le 1} \vct{g}^T\vct{x}$.
We also need the definition of the mirror map and Bregman divergence (our exposition is adapted from \cite{bubeck2015convex}).
\begin{definition}[mirror map] Let $\mathcal{D}\subset\R^n$ be a convex open set. We say that $\Phi:\mathcal{D}\rightarrow \R$ is a mirror map if it satisfies the following properties: \begin{itemize}
\item[(a)]  $\Phi$ is strictly convex and differentiable.
\item[(b)] The gradient of $\Phi$ takes all possible values, that is $\nabla \Phi(\mathcal{D})=\R^n$.
\item[(c)] The gradient of $\Phi$ diverges on the boundary of $\mathcal{D}$, that is $\underset{\vct{x}\rightarrow \partial \mathcal{D}}{\lim}\text{ }\|\nabla \Phi(\vct{x})\|=+\infty.$ We study mirror maps with $\mathcal{D}=\R_{+}^n$ equal to the positive orthant.
\end{itemize}
%
\end{definition}
\begin{definition}[Bregman Divergence and Projection] We define the Bregman divergence associated to a mirror map $\phi$ as $$D_{\phi}(\vct{x},\vct{y})=\phi(\vct{x})-\phi(\vct{y})-\langle \nabla \phi(\vct{y}),\vct{x}-\vct{y}\rangle.$$
We also define the projection onto a set $\mathcal{K}$ with respect to a mapping $\Phi$ via $$\Pi_{\mathcal{K}}^{\Phi}(\vct{y})=\underset{\vct{x}\in\mathcal{K}}{\arg\min}\text{ }\mathcal{D}_{\Phi}(\vct{x},\vct{y}).$$
\end{definition}
Finally, we define the diameter as follows $$R^2=\underset{\vct{x}\in\mathcal{K}, \vct{y}\in\mathcal{K}\cap\R_{++}^n}{\sup} \Phi(\vct{x})-\Phi(\vct{y}).$$
We are now ready to describe the stochastic mirror method based on a mirror map $\Phi$ and let $\vct{x}_1\in\arg\min_{\vct{x}\in\mathcal{K}} \Phi(\vct{x})$. Also, let $\vct{g}_t$ be an unbiased estimator of the gradient $\nabla f(\vct{x}_t)$. Then for $t\ge1$, let $\vct{y}_{t+1}\in\R_{+}^n$ be such that $\nabla \Phi(\vct{y}_{t+1})=\nabla \Phi(\vct{x}_{t})-\mu_t\vct{g}_t$. Using $\vct{y}_{t+1}$ we obtain the next estimate $\vct{x}_{t+1}$ by projecting $\vct{y}_{t+1}$ onto $\mathcal{K}$ using the mirror map. We detail the Stochastic Mirror Ascent in Algorithm \ref{alg:smm}.
  \begin{algorithm}[t]
 \caption{(Stochastic) Mirror Ascent for Maximizing $F(x)$ over a convex set $\C$}
\label{alg:smm}
\begin{algorithmic}
\STATE {\bfseries Parameters:} Integer $T >0$ and scalars $\mu_t >0$, $t \in [T]$
\STATE{\bfseries Initialize:} $x_1  \in \C$
\FOR{$t=1$ \textbf{to} $T$}
\STATE $\nabla \phi(\y_{t+1}) \leftarrow \nabla \phi(\x_t) + \mu_t \vct{g}_t$ with $\vct{g}_t$ obeying $\mathbb{E}[ \vct{g}_t | \x_t] = \nabla F(\x_t)$
\STATE $\x_{t+1} = {\arg\min}_{\vct{x}\in\mathcal{K}}\text{ }\mathcal{D}_{\Phi}(\vct{x},\vct{y}_{t+1})$ where $\mathcal{D}_{\Phi}$ is the Bregman divergence associated with the mirror map
 \ENDFOR 
\STATE{Pick $\tau$ uniformly at random from $\{1,2,\ldots,T\}$.}
\STATE{\bfseries Output} $x_\tau$ 
\end{algorithmic}
  \end{algorithm}

\begin{theorem}[Stochastic Mirror Method]\label{thm:mirror} Let $\Phi$ be a mirror map that is $1$-strongly convex on $\mathcal{K}$ with respect to the norm $\|\|$. 
Assume that $F$ is $L$-smooth with respect to the norm $\|\|$ and is a monotone, continuous submodular function. Furthermore, assume that we have access to a stochastic oracle $\Gt$ obeying $$\E[\Gt]=\nabla F(\x_t)\quad\text{and}\quad \E\big[\|\Gt-\nabla F(\x_t)\|_{*}^2\big]\le \sigma^2.$$ We start from  $\vct{x}_1\in\arg\min_{\vct{x}\in\mathcal{K}} \Phi(\vct{x})$ and run the mirror ascent updates of the form
\begin{eqnarray*}
\nabla \Phi(\y_{t+1})&=&\nabla \Phi(\x_{t})+\mu_t \Gt,\\
\quad\x_{t+1}&=&\Pi_{\mathcal{K}}^{\Phi}\left(\y_{t+1}\right),
\end{eqnarray*}
with $\mu_t=\frac{1}{L+\frac{\sigma}{R}\sqrt{t}}$. Let $\tau$ be a random variable taking values in $\{1,2,\ldots,T\}$ with equal probability. Then,
\begin{align}
\label{optguarantee}
\E[F(\x_\tau)] \ge  \frac{\rm{OPT}}{2}- \left(\frac{R^2L+\text{OPT}}{2T}+\frac{R\sigma}{\sqrt{T}}\right).
\end{align}
\end{theorem} 
\begin{remark}
We would like to note that if we pick $\tau$ to be a random variable taking values in $\{2,\ldots,T-1\}$ with probability $\frac{1}{(T-1)}$ and $1$ and $T$ each with probability $\frac{1}{2(T-1)}$ then
\begin{align*}
\E[F(\x_\tau)] \ge  \frac{\rm{OPT}}{2} -\left(\frac{R^2L}{2T}+\frac{R\sigma}{\sqrt{T}}\right).
\end{align*}
\end{remark}
As a simple application of Theorem~\ref{thm:mirror}, let us consider submodular optimization problems that arise from maximizing submodular set functions under  $k$-cardinality constraints. For such problems, it will be convenient to use the mirror ascent method with $\ell_1$ norm on the scaled simplex $\{\vct{z}\in [0,1]^n: \sum_{i=1}^n\vct{z}_i=k\}$ with the entropy mirror map $\Phi(\vct{x})=k\sum_{i=1}^nx(i)\log x(i)$. This is due to the fact that the smoothness parameter of the multilinear extension in the $\ell_1$ norm might be much smaller than its counterpart in the $\ell_1$ norm.   In this case, the updates in Algorithm~\ref{alg:smm} take the form
\begin{eqnarray*}
[\vct{y}_{t+1}]_i&=&[\vct{x}_{t}]_ie^{-\eta[\frac{\mu_t}{k}\vct{g}_t]_i},\text{ for }i=1,2,\ldots,n,\quad\\ \vct{x}_{t+1}&=&\underset{\vct{x}\in\mathcal{K}}{\arg\min}\text{ KL}(\vct{x},\vct{y}_{t+1}).
\end{eqnarray*}
Here, KL$(\vct{x},\vct{y})$ denotes the KL divergence between the two vectors $\vct{x}$ and $\vct{y}$. We also note that the corresponding projection can be be done very efficiently in $\mathcal{O}(n)$ time using standard methods described in \cite{brucker1984n, pardalos1990algorithm}. The reason for using  the $\ell_1$ norm with the entropy map is that the smoothness parameter $L$ can be bounded by a constant. Indeed,  the smoothness parameter of the multilinear extension of a monotone submodular function $f$  can be bounded by the maximum marginal value of $f$, i.e. $L\le m _f\triangleq \max_{e} f(e)$. 
%
Furthermore, $R^2$ can be bounded by $O(k\log n)$. Thus, using this particular mirror method, the above result roughly states that $T=\mathcal{O}\left(\frac{m_f k\log(n)}{\epsilon}+\frac{k\sigma^2\log(n)}{\epsilon^2}\right)$
%
iterations of the stochastic mirror method, yields a solution whose objective value is at least $\frac{\text{OPT}}{2} - \epsilon$. Stated differently, $T=\mathcal{O}\left(\frac{m_fk\log(n)}{\epsilon}+\frac{k\sigma^2\log(n)}{\epsilon^2}\right)$ iterations of the stochastic mirror method provides in expectation an objective value exceeding $\frac{\text{OPT}}{2} - \epsilon$ approximation ratio for DR-submodular maximization. Thus, the required number of iterations to reach a certain accuracy now depends only logarithmically on  $n$.

\section{Experiments} \label{sec:experiments}
\begin{figure*}[t!]
  \centering
  \begin{subfigure}{0.45\textwidth}
    \begin{center}
      \centerline{\includegraphics[width=\columnwidth]{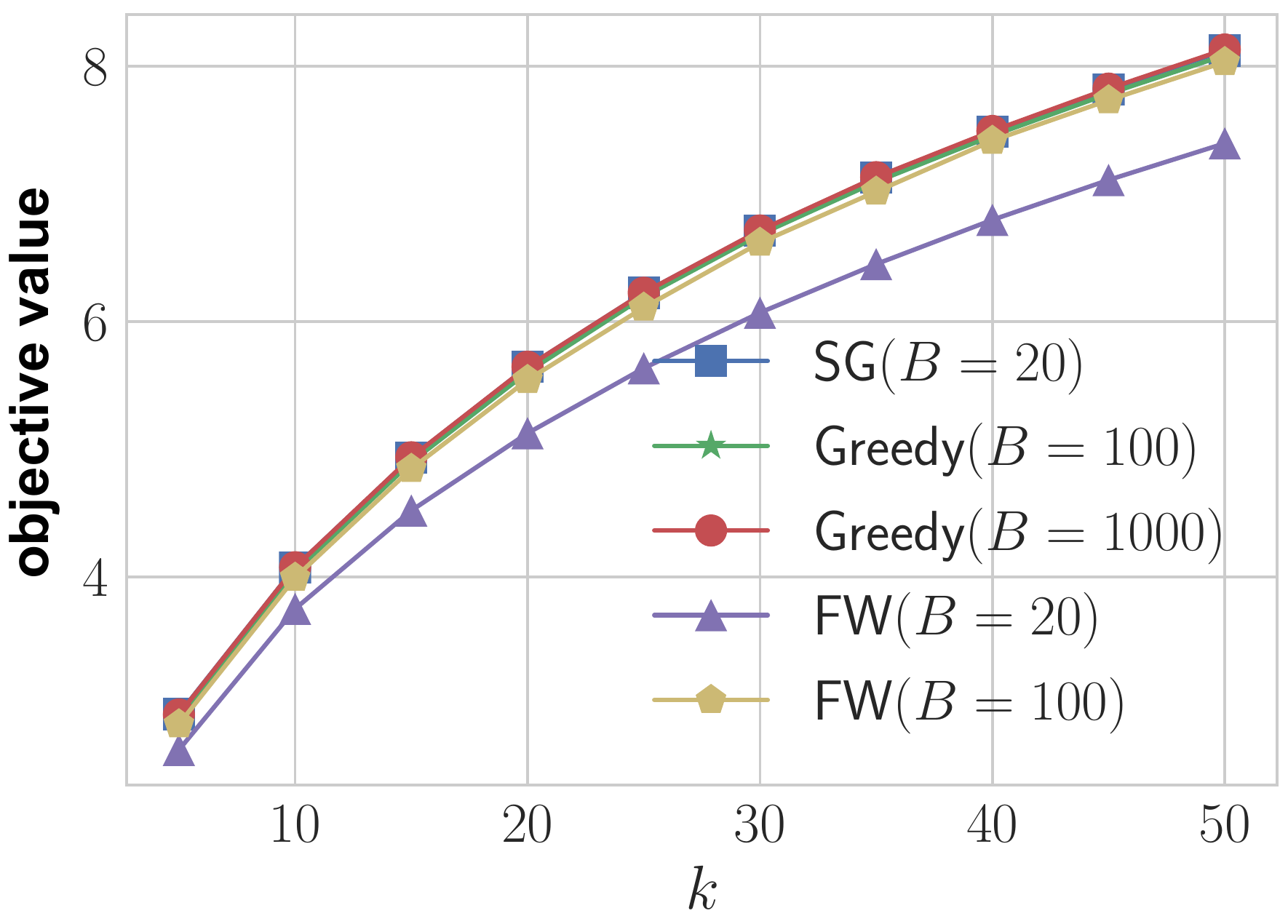}}
      \caption{Concave Over Modular}
      \label{fig1}
    \end{center}
  \end{subfigure}%
  \begin{subfigure}{0.45\textwidth}
    \begin{center}
      \centerline{\includegraphics[width=\columnwidth]{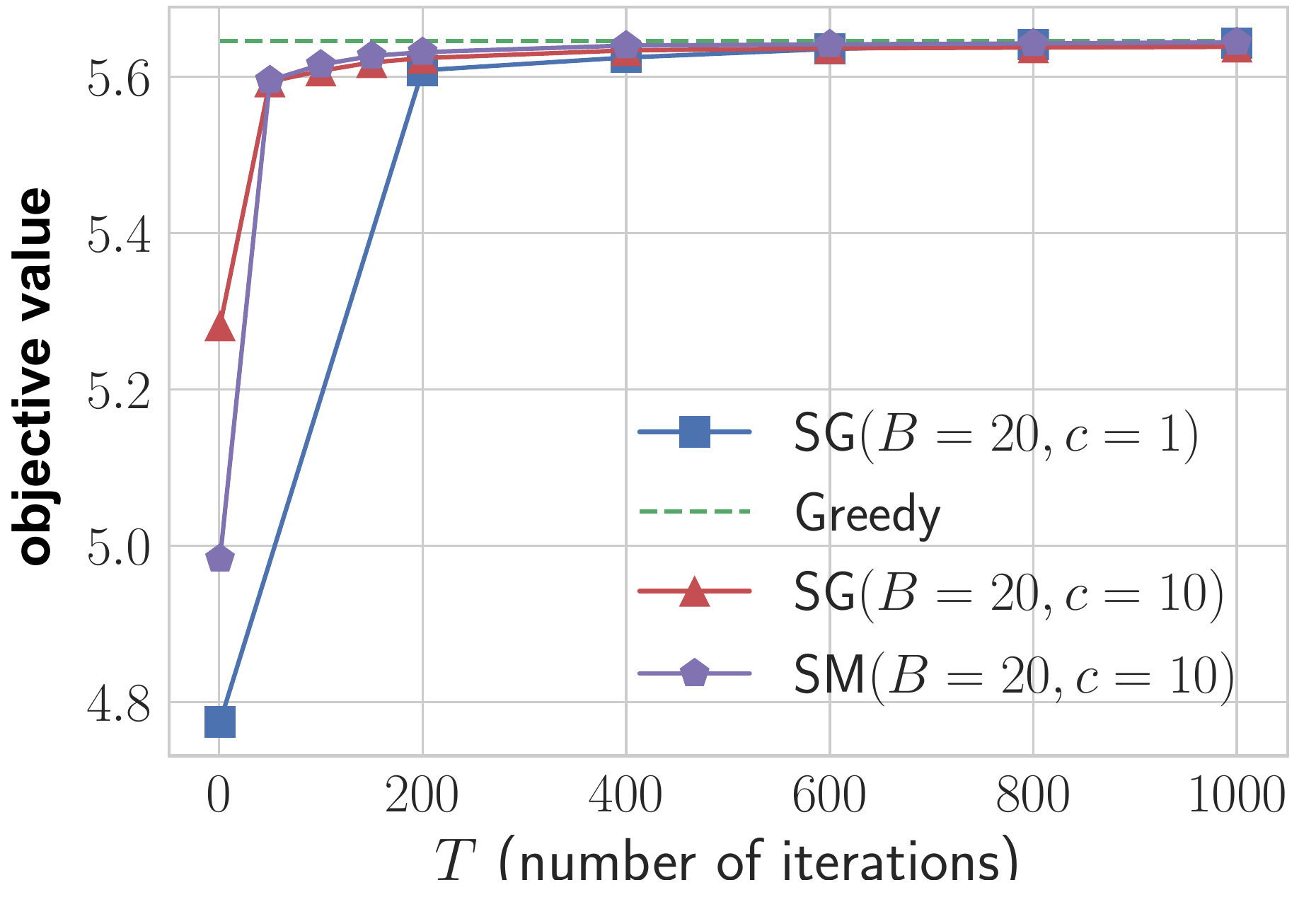}}
              \caption{Concave Over Modular}
      \label{fig2}
    \end{center}
    
  \end{subfigure}\\
    \begin{subfigure}{0.45\textwidth}
    \begin{center}
      \centerline{\includegraphics[width=\columnwidth]{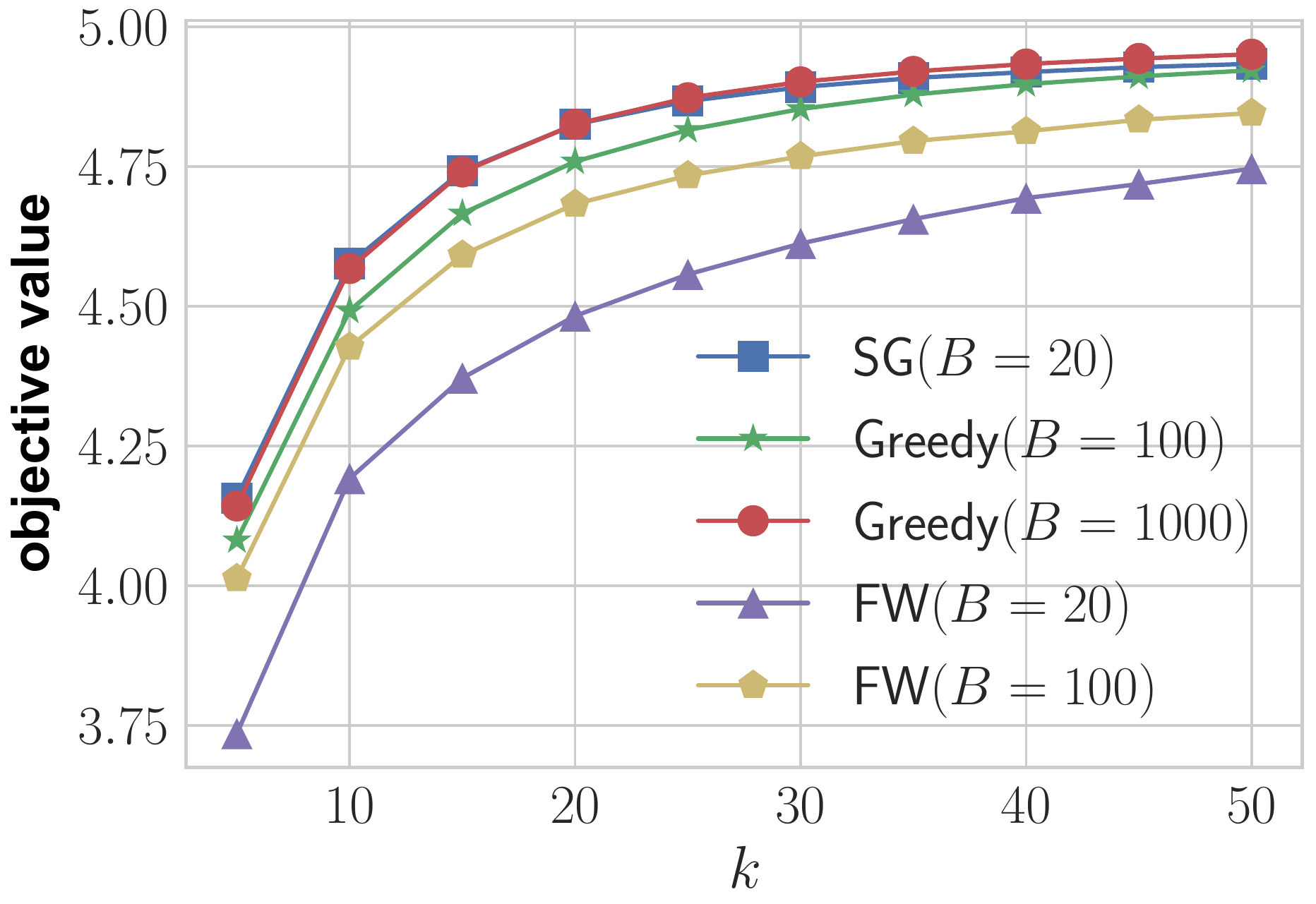}}
            \caption{Facility Location}
      \label{fig3}
    \end{center}
      \end{subfigure}
        \begin{subfigure}{0.45\textwidth}
    \begin{center}
      \centerline{\includegraphics[width=\columnwidth]{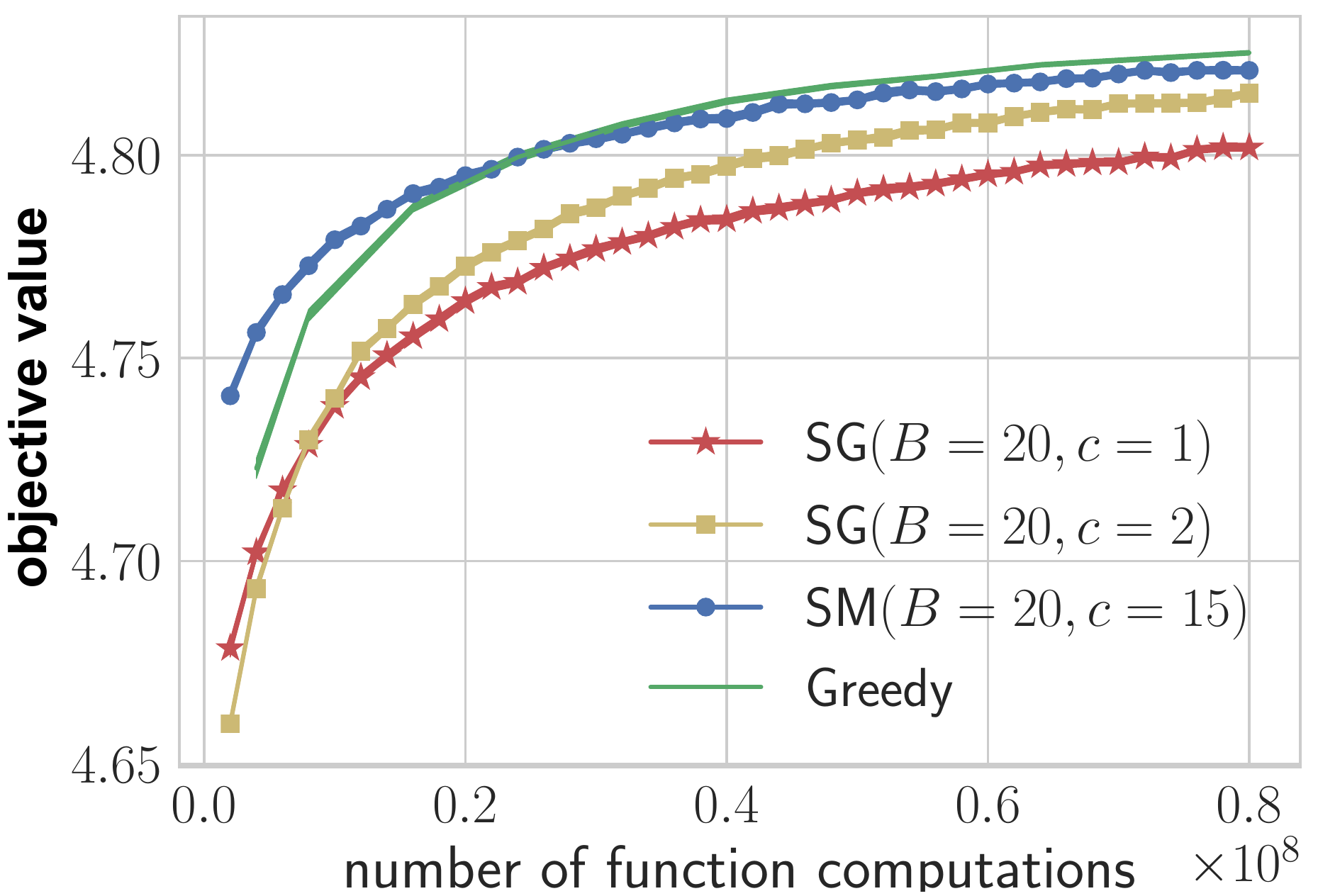}}
            \caption{Facility Location}
      \label{fig4}
    \end{center}
  \end{subfigure}~%
    \caption{(a) shows the performance of the algorithms  w.r.t. the cardinality constraint $k$ for the concave over modular  objective.  Each of the continuous algorithms (i.e., SG, SM and FW)  run for $T = 2000$ iterations. (b) shows the performance of the algorithms SG and SM versus the number of iterations for fixed $k=20$ for the concave over modular objective. The green dashed line indicates the value obtained by Greedy (with $B = 1000$). Recall that the step size of SM and SG is $c/\sqrt{t}$. (c) shows the performance of the algorithms  w.r.t. the cardinality constraint $k$ for the facility location  objective function.  Each of the continuous algorithms (SG, SM, FW)  run for $T = 2000$ iterations. (d) shows the performance of different algorithms versus  the number of simple function computations (i.e. the number of $f_i$'s evaluated during the algorithm) for the facility location objective function. For the greedy algorithm, larger number of function computations corresponds to a larger batch size. For SG and SM, larger time corresponds to larger iterations.   \label{fig-kolli}}    
\end{figure*}
In our experiments, we consider a movie recommendation application \cite{serban17} where each user $i$ has a user-specific utility function $f_i$ for evaluating sets of movies. The goal is to find a set of $k$ movies such that in expectation over users' preferences it provides the highest utility, i.e., $\max_{|S|\leq k}f(S)$, where $f(S)  \doteq\E_{i\sim\dist}[f_i(S)]$.  This is an instance of the stochastic submodular maximization problem defined in \eqref{eq:stochsub}.  One solution to this problem is to sample $B$ users and consider the empirical objective function $\frac{1}{B}\sum_{j=1}^B f_i$. When $B$ is large enough, this provides a good estimate of the true objective function $f$. We can then run the (discrete) greedy algorithm on the  empirical objective function  to find a good set of size $k$. Another way of solving this problem is to evaluate the multilinear extension $F_i$ of any sampled function $f_i$ and solve the problem in the continuous domain as follows. Let $F(\vct{x}) = \mathbb{E}_{i \sim \dist} [F_i(\vct{x})]$ for $x \in [0,1]^n$ and define the constraint set $\mathcal{P}_k =   \{ \vct{x} \in [0,1]^m: \sum_{i=1}^n x_i \leq k\}$. The discrete and continuous optimization formulations lead to the same optimal value \cite{calinescu11maximizing}: $$\max_{S: |S| \leq k} f(S) = \max_{\vct{x} \in \mathcal{P}_k} F(\vct{x}).$$
Therefore, by running the stochastic versions of projected gradient methods, we can find a solution in the continuous domain that is at least $1/2$ approximation to the optimal value. By rounding that fractional solution (for instance via randomized Pipage rounding \cite{calinescu11maximizing}) we obtain a set whose utility is at least $1/2$ of the optimum solution set of size $k$. We note that randomized Pipage rounding does not need access to the value of $f$. We also remark that projection onto  $\mathcal{P}_k$ can be done very efficiently in $O(n)$ time (see \cite{karimi17, brucker1984n, pardalos1990algorithm}). Therefore, such approach easily scales to big data scenarios where the size of the data set (e.g.~number of users) or the number of items $n$ (e.g.~number of movies) are very large.

\noindent In our experiments, we consider the following baselines: 
\begin{itemize}
 \item[(i)] Stochastic Gradient Ascent (SG): with the step size $\mu_t = c/\sqrt{t}$ and batch size $B$. The details for computing an unbiased estimation for the gradient of $F$ are given in Appendix \ref{unbiased}.  
  \item[(ii)] Stochastic Mirror Ascent (SM): with the step size $\mu_t = c/\sqrt{t}$ and batch size $B$. 
 \item[(iii)] Frank-Wolfe (FW) variant of \cite{bian2016guaranteed}: with parameter $T$ for the total number of iterations and batch size $B$  (we further let $\alpha =1, \delta=0$, see Algorithm 1 in \cite{bian2016guaranteed} for more details). 		
 \item [(iv)] Batch-mode Greedy (Greedy): by running greedy algorithm over the empirical objective function with  $B$ samples. 
\end{itemize}

To run the experiments we use the MovieLens data set. It consists of 1 million ratings (from 1 to 5) by $n=6041$ users for $m=4000$ movies.  Let $r_{i,j}$ denote the rating of user $i$ for movie $j$ (if such a rating does not exist we assign $r_{i,j}$ to 0). In our experiments, we consider two well motivated objective functions. The first one is the facility location where the valuation function by user $i$ is defined as  $f_i(S) = \max_{j\in S} r_{i,j}$. In words, the way user $i$ evaluates a set $S$ is by picking the highest rated movie in $S$. For simplicity, we also assume that the distribution $\dist$ is uniform. Thus, the objective function is $f_1(S) = \frac{1}{n}\sum_{i=1}^n  \max_{j\in S} r_{i,j}$.  

In our second experiment, we consider a  different user-specific valuation function which is a concave function composed with a modular function, i.e.,  $f_i(S) =(\sum_{j\in S} r_{i,j})^{1/2}.$ Again, by considering the uniform distribution over the set of users, we obtain $f_2(S) = \frac{1}{n} \sum_{i=1}^n  (\sum_{j\in S} r_{i,j})^{1/2}.$ Note that the multilinear extensions of $f_1$ and $f_2$ are neither concave nor convex. 

Figure~\ref{fig-kolli} depicts the performance of different algorithms for the two proposed objective functions.  As Figures~\ref{fig1} and  \ref{fig3} show, the FW algorithm needs a much higher batch size to be comparable in performance  w.r.t. to our stochastic gradient methods. With the same batch size and number of iterations SG, SM, and FW have similar computational complexity. Therefore, a smaller batch size leads to less computational effort. Figure~\ref{fig2} shows that after a few hundred iterations both SG and SM with $B=20$ obtain almost the same utility as Greedy with a large batch size ($B=1000$). Finally, Figure~\ref{fig4} shows the performance of the algorithms with respect to the number of times the single functions ($f_i$'s) are evaluated.  This further shows that gradient based methods have comparable complexity w.r.t. the Greedy algorithm in the discrete domain.  

%

\section{Conclusion}
In this paper we studied gradient methods for submodular maximization. Despite the lack of convexity of the objective function we demonstrated that local search heuristics are effective at finding approximately optimal solutions. In particular, we showed that all fixed point of projected gradient ascent provide a factor $1/2$ approximation to the global maxima. We also demonstrated that stochastic gradient and mirror methods achieve an objective value of ${\rm OPT}/2-\epsilon$ in $\mathcal{O}(\frac{1}{\epsilon^2})$ iterations. We further demonstrated the effectiveness of our methods with experiments on real data. 

While in this paper we have focused on convex constraints, our framework may allow non-convex constraints as well. For instance it may be possible to combine our framework with recent results in \cite{oymak2015sharp, soltanolkotabi2017structured} to deal with general nonconvex constraints. Furthermore, in some cases projection onto the constraint set may be computationally intensive or even intractable but calculating an approximate projection may be possible with significantly less effort. One of the advantages of gradient descent-based proofs is that they continue to work even when some perturbations are introduced in the updates. Therefore, we believe that our framework can deal with approximate projections and we hope to pursue this in future work.
%

\section{Proofs}
Throughout this section we use the assumption that $\mathcal{K}\subseteq \mathcal{X}\subset\R_{+}^n$. We first prove Theorems~\ref{local_opt} and \ref{thm:mirror}. We then show in Section \ref{proofofavg} how the proof of Theorem~\ref{thm:sgd} follows from the proof of Theorem~\ref{thm:mirror}.
\subsection{Proofs for quality of stationary points (Proof of Theorem~\ref{local_opt})}
Let us begin with part (i) of the theorem. Let $F: \mathcal{X} \to \mathbb{R}_+$ be a weakly DR-submodular and monotone function. By \eqref{eq:weaksub} for any two vectors $\vct{x},\vct{y}\in \mathcal{X}$ we have
\begin{align}
\label{cond}
\nabla F(\vct{x})\ge \gamma \nabla F(\vct{y})\quad\text{for all}\quad\vct{x}\preceq\vct{y}.
\end{align}
To prove part (i) we first prove that for any two vectors $\vct{x},\vct{y}\in \mathcal{X},$ we have
\begin{align}
\label{mahdicond}
F(\vct{y})-\left(1+\frac{1}{\gamma^2}\right)F(\vct{x})\le \frac{1}{\gamma}\langle\nabla F(\vct{x}),\vct{y}-\vct{x}\rangle.
\end{align}
To this aim note that for all $\vct{x}, \vct{z} \in \X$ s.t. $\vct{x}\preceq\vct{z}$, by using \eqref{cond}, we have
\begin{align}
\label{ineq1}
F(\vct{z})-F(\vct{x})=&\int_0^1 \langle \vct{z}-\vct{x},\nabla F\left(\vct{x}+t(\vct{z}-\vct{x})\right)\rangle dt,\nonumber\\
\le& \frac{1}{\gamma}\int_0^1 \langle \vct{z}-\vct{x},\nabla F(\vct{x})\rangle dt,\nonumber\\
=& \frac{1}{\gamma}\langle \vct{z}-\vct{x},\nabla F(\vct{x})\rangle.
\end{align}
Similarly, note that using \eqref{cond}
\begin{align}
\label{ineq2}
F(\vct{z})-F(\vct{x})=&\int_0^1 \langle \vct{z}-\vct{x},\nabla F\left(\vct{x}+t(\vct{z}-\vct{x})\right)\rangle dt,\nonumber\\
\ge& \gamma\int_0^1 \langle \vct{z}-\vct{x},\nabla F(\vct{z})\rangle dt,\nonumber\\
=& \gamma \langle \vct{z}-\vct{x},\nabla F(\vct{z})\rangle.
\end{align}
Now, from \eqref{ineq1} we deduce that for any $\vct{x} , \vct{y}  \in \X$:
\begin{align*}
F(\vct{x} \vee \vct{y} )-F(\vct{x}) \leq \frac{1}{\gamma}\langle \vct{x} \vee \vct{y}-\vct{x},\nabla F(\vct{x})\rangle,
\end{align*} 
and from \eqref{ineq2}:
\begin{align*}
F(\vct{x}) - F(\vct{x} \wedge \vct{y} ) \geq \gamma \langle \vct{x} -  \vct{x} \wedge \vct{y},\nabla F(\vct{x})\rangle,
\end{align*} 
From these two inequalities we immediately obtain
\begin{align}
F(\vct{x} \vee \vct{y}) - (1 + \frac{1}{\gamma^2} ) F(\vct{x} ) + \frac{1}{\gamma^2} F(\vct{x} \wedge \vct{y})  \leq \frac{1}{\gamma} \langle   \vct{x} \wedge \vct{y} + \vct{x} \vee \vct{y} - 2 \vct{x},\nabla F(\vct{x}) \rangle,
\end{align} 
and we obtain \eqref{mahdicond} by noting that $F(\vct{x} \wedge \vct{y}) \geq 0$ and $ \vct{x} \wedge \vct{y} + \vct{x} \vee \vct{y}  = \vct{x} + \vct{y}$.

Part (i) of the theorem follows from \eqref{mahdicond} by letting $\vct{x}$ to be a stationary point and $\vct{y}=\vct{x}^*:=\underset{\mathcal{K}}{\arg\max} F(\vct{y})$.

To prove part (ii) note that by the smoothness of the function (more specifically the quadratic upper bound) we have
\begin{align*}
F(\vct{x}_{t+1})\ge F(\vct{x}_t)+\langle\nabla F(\vct{x}_t),\vct{x}_{t+1}-\vct{x}_t\rangle-\frac{L}{2}\twonorm{\vct{x}_{t+1}-\vct{x}_t}^2.
\end{align*}
Now note that $\vct{x}_{t+1}=\mathcal{P}_{\mathcal{K}}\left(\vct{x}_t+\mu_t\nabla F(\vct{x}_t)\right)$ and thus using the properties of convex projections we have
\begin{align*}
\langle\vct{x}_{t+1}-\vct{x}_t,\vct{x}_{t+1}-\left(\vct{x}_t+\mu_t\nabla F(\vct{x}_t)\right)\rangle\le0\quad\Rightarrow\quad\twonorm{\vct{x}_{t+1}-\vct{x}_t}^2\le \mu_t\langle \vct{x}_{t+1}-\vct{x}_t,\nabla F(\vct{x}_t)\rangle.
\end{align*}
Plugging this into the latter inequality we conclude that for $\mu_t\le \frac{1}{L}$
\begin{align*}
F(\vct{x}_{t+1})\ge F(\vct{x}_t)+\left(\frac{1}{\mu_t}-\frac{L}{2}\right)\twonorm{\vct{x}_{t+1}-\vct{x}_t}^2\ge F(\vct{x}_t)+\frac{L}{2}\twonorm{\vct{x}_{t+1}-\vct{x}_t}^2.
\end{align*}
Summing both sides we conclude that
\begin{align*}
\sum_{t=1}^{\infty} \twonorm{\vct{x}_{t+1}-\vct{x}_t}^2,
\end{align*}
is bounded. This in turn implies that $\twonorm{\vct{x}_{t+1}-\vct{x}_t}$ goes to zero. Which implies that $\vct{x}_t$ converges to a point $\vct{x}$. This means that this point obeys
\begin{align*}
\vct{x}=\mathcal{P}_{\mathcal{K}}\left(\vct{x}+\mu_t\nabla F(\vct{x})\right).
\end{align*}
By definition of projection the latter implies that $\mathcal{P}_{\mathcal{K}-\{\vct{x}\}}(\mu_t\nabla F(\vct{x}))=0$. A well known result in convex analysis (e.g.,~see \cite[Lemma 6.4]{oymak2015sharp} or \cite[Lemma 7.11]{soltanolkotabi2017structured}) implies $\max_{\vct{y} \in \mathcal{K}} \langle  \nabla F(x), \vct{y} - \vct{x} \rangle \leq 0$, concluding the proof.
\subsection{Proof of (stochastic) mirror method (Proof of Theorem~\ref{thm:mirror})}
\label{proofofavg}
We begin by stating some lemmas about mirror descent together with some useful preliminary lemmas in the next section.
\subsubsection{Preliminary lemmas}\label{backgnd}
We begin with two lemmas about mirror descent adapted from \cite{bubeck2015convex}.
\begin{lemma}\label{cvxProj} Let $\vct{x}\in\mathcal{K}$ and $\vct{y}\in\mathcal{K}$, then
\begin{align*}
\big\langle \nabla \Phi\left(\Pi_{\mathcal{K}}^{\Phi}(\y)\right)-\nabla\Phi(\y),\Pi_{\mathcal{K}}^{\Phi}(\y)-\x\big\rangle\le 0.
\end{align*}
\end{lemma}
Also we need the following well-known identity about Bregman divergences which will be useful several times in our proofs.
\begin{lemma}\label{3sum}
\begin{align*}
\langle \nabla \phi(\x)-\nabla \phi(\y),\x-\z\rangle=\Dphi(\x,\y)+\Dphi(\z,\x)-\Dphi(\z,\y).
\end{align*} 
\end{lemma}
We next state a lemma due to Chekuri, Vondrak, and Zenkluser.
\begin{lemma}\cite[Lemma 3.2]{vondrak2011submodular}\label{Vondrak} Assume $F$ is a monotone and submodular function. Then, for any two points $\vct{x},\vct{y}\in\mathcal{K}$
\begin{align*}
\big\langle \x-\y,\nabla F(\x)\big\rangle\le 2F(\x)-F(\max(\x,\y))-F(\min(\x,\y)).
\end{align*}
\end{lemma}
\begin{lemma}\label{fu} Consider one iteration of the mirror descent update
\begin{align*}
\nabla \Phi(\y)=&\nabla \Phi(\x)+\mu G(\vct{x}),\\
\x^{+}=&\Pi_{\mathcal{K}}^{\Phi}\left(\y\right).
\end{align*}
Then, for all $\vct{z}\in\mathcal{K}$
\begin{align*}
\langle G(\vct{x}), \x^{+}-\z\rangle\ge\frac{1}{\mu}\left( \Dphi(\x^{+},\x)+\Dphi(\z,{\x}^{+})-\Dphi(\z,\x)\right).
\end{align*}
\end{lemma}
\begin{proof}
By Lemma \ref{cvxProj}
\begin{align*}
\big\langle \nabla\Phi(\x^{+})-\nabla \Phi(\y),\x^{+}-\z\big\rangle\le 0.
\end{align*}
Using $\nabla \Phi(\y)=\nabla \Phi(\x)+\mu G(\vct{x})$, we conclude that
\begin{align*}
\big\langle G(\vct{x}),\x^{+}-\z\big\rangle\ge & \frac{1}{\mu} \big\langle \Phi(\x^{+})-\nabla \Phi(\x), \x^{+}-\z\big\rangle,\\
=&\frac{1}{\mu}\left( \Dphi(\x^{+},\x)+\Dphi(\z,{\x}^{+})-\Dphi(\z,\x)\right),
\end{align*}
where the last equality follows from Lemma \ref{3sum}.
\end{proof}
\begin{lemma}\label{key} Consider the setting of Theorem~\ref{thm:mirror} and let $D_\Phi$ be the Bregman divergence corresponding to the mirror map $\Phi$. Let $\eta$ be a nonnegative scalar with $\mu_t$ obeying
\begin{align*}
\mu_t\le \frac{1}{L+\frac{1}{\eta}}.
\end{align*}
 Then, 
\begin{align*}
\big\langle \Gt, \x_t-\z\big\rangle\ge \frac{1}{\mu_t} \left(\Dphi(\z,\x_{t+1})-\Dphi(\z,\x_{t})\right)+F(\x_t)-F(\x_{t+1})-\frac{\eta}{2}\|\nabla F(\x_t)-\Gt\|_{*}^2.
\end{align*}
\end{lemma}
\begin{proof}
Using smoothness of the function $F$ we have the following chain of inequalities
\begin{align*}
F(\x_{t+1})-F(\x_{t})&\ge \langle\nabla F(\x_t),\x_{t+1}-\x_t\rangle-\frac{L}{2}\|\x_{t+1}-\x_t\|^2\\
&=\langle\Gt,\x_{t+1}-\x_t\rangle+\langle \nabla F(\x_t)-\Gt,\x_{t+1}-\x_t\rangle-\frac{L}{2}\|\x_{t+1}-\x_t\|^2\\
&\overset{(a)}{\ge}\langle\Gt,\x_{t+1}-\x_t\rangle-\frac{\eta}{2}\|\nabla F(\x_t)-\Gt\|_{*}^2-\frac{1}{2}\left(L+\frac{1}{\eta}\right)\|\x_{t+1}-\x_t\|^2\\
&\overset{(b)}{\ge}\langle\Gt,\x_{t+1}-\x_t\rangle-\frac{\eta}{2}\|\nabla F(\x_t)-\Gt\|_{*}^2-\left(L+\frac{1}{\eta}\right)\Dphi(\x_{t+1},\x_t)\\
&=\langle\Gt,\z-\x_t\rangle+\langle\Gt,\x_{t+1}-\z\rangle-\frac{\eta}{2}\|\nabla F(\x_t)-\Gt\|_{*}^2-\left(L+\frac{1}{\eta}\right)\Dphi(\x_{t+1},\x_t),
\end{align*}
where (a) follows from the fact that $\langle\vct{a},\vct{b}\rangle\le \frac{1}{2\eta}\|\vct{a}\|^2+\frac{\eta}{2}\|\vct{b}\|_{*}^2$ by Young's inequality and (b) follows from strong convexity of the mirror map $\Phi$. Rearranging the above inequality we arrive at the following chain of inequalities
\begin{align*}
\langle\Gt,\x_t-\z\rangle\ge&\langle\Gt,\x_{t+1}-\z\rangle-\frac{\eta}{2}\|\nabla F(\x_t)-\Gt\|_{*}^2-\left(L+\frac{1}{\eta}\right)\Dphi(\x_{t+1},\x_t)+F(\x_t)-F(\x_{t+1})\\
\overset{(a)}{\ge}&\frac{1}{\mu_t}\left(\Dphi(\z,\x_{t+1})-\Dphi(\z,\x_t)\right)-\frac{\eta}{2}\|\nabla F(\x_t)-\Gt\|_{*}^2+\left(\frac{1}{\mu_t}-\left(L+\frac{1}{\eta}\right)\right)\Dphi(\x_{t+1},\x_t)\\
&+F(\x_t)-F(\x_{t+1})\\
\overset{(b)}{\ge}&\frac{1}{\mu_t}\left(\Dphi(\z,\x_{t+1})-\Dphi(\z,\x_t)\right)-\frac{\eta}{2}\|\nabla F(\x_t)-\Gt\|_{*}^2+F(\x_t)-F(\x_{t+1}),
\end{align*}
where (a) follows from Lemma \ref{fu} and (b) from the choice $\mu_t\le 1/\left(L+1/\eta\right)$.

\end{proof}

Using Lemma \ref{key} with $\z=\vct{x}^*$ (Global optimum) and $\eta=\eta_t$ we have
\begin{align*}
\big\langle \Gt, \x_t-\x^{*}\big\rangle\ge& \frac{1}{\mu_t} \left(\Dphi(\x^{*},\x_{t+1})-\Dphi(\x^{*},\x_{t})\right)+F(\x_t)-F(\x_{t+1})-\frac{\eta_t}{2}\|\nabla F(\x_t)-\Gt\|_{*}^2.
\end{align*}
Using $\big\langle \Gt, \x_t-\x^{*}\big\rangle=\big\langle \nabla F(\x_t), \x_t-\x^{*}\big\rangle+\big\langle \Gt-\nabla F(\x_t), \x_t-\x^{*}\big\rangle$ we conclude that
\begin{align*}
\big\langle \nabla F(\x_t), \x_t-\x^{*}\big\rangle\ge& \big\langle \Gt-\nabla F(\x_t), \x_t-\x^{*}\big\rangle\\
&+\frac{1}{\mu_t} \left(\Dphi(\x^{*},\x_{t+1})-\Dphi(\x^{*},\x_{t})\right)+F(\x_t)-F(\x_{t+1})-\frac{\eta_t}{2}\|\nabla F(\x_t)-\Gt\|_{*}^2.
\end{align*}
Using Lemma \ref{Vondrak} with $\y=\x^*$ and $\x=\x_t$ in the above inequality we conclude that
\begin{align*}
F(\x_t)+F(\x_{t+1})-F({\x}^*)\ge& \big\langle \Gt-\nabla F(\x_t), \x_t-\x^{*}\big\rangle\\
&+\frac{1}{\mu_t} \left(\Dphi(\x^{*},\x_{t+1})-\Dphi(\x^{*},\x_{t})\right)-\frac{\eta_t}{2}\|\nabla F(\x_t)-\Gt\|_{*}^2.
\end{align*}
Taking expectation of both sides we arrive at
\begin{align*}
\E F(\x_t)+\E F(\x_{t+1})-F({\x}^*)\ge& \frac{1}{\mu_t} \left(\E\Dphi(\x^{*},\x_{t+1})-\E\Dphi(\x^{*},\x_{t})\right)-\frac{\eta_t}{2}\E\big[\|\nabla F(\x_t)-\Gt\|_{*}^2\big],\\
\ge& \frac{1}{\mu_t} \left(\E\Dphi(\x^{*},\x_{t+1})-\E\Dphi(\x^{*},\x_{t})\right)-\frac{\eta_t}{2}\sigma^2.
\end{align*}
Summing both sides from $t=1$ to $T$ we conclude that
\begin{align*}
\sum_{t=1}^T\big[\E F(\x_t)+\E F(\x_{t+1})-F({\x}^*)\big]\ge& \sum_{t=1}^T \frac{1}{\mu_t} \left(\E\Dphi(\x^{*},\x_{t+1})-\E\Dphi(\x^{*},\x_{t})\right)-\frac{\sigma^2}{2}\sum_{t=1}^T\eta_t\\
=&\frac{\E\Dphi(\x^{*},\x_{T+1})}{\mu_T}-\frac{\E\Dphi(\x^{*},\x_{1})}{\mu_1}+\sum_{t=1}^{T-1} \E\Dphi(\x^{*},\x_{t+1})\left(\frac{1}{\mu_t}-\frac{1}{\mu_{t+1}}\right)\\
&-\frac{\sigma^2}{2}\sum_{t=1}^T\eta_t\\
\overset{(a)}{\ge}&-\frac{R^2}{\mu_1}+R^2\sum_{t=1}^{T-1}\left(\frac{1}{\mu_t}-\frac{1}{\mu_{t+1}}\right)-\frac{\sigma^2}{2}\sum_{t=1}^T\eta_t\\
=&-\frac{R^2}{\mu_T}-\frac{\sigma^2}{2}\sum_{t=1}^T\eta_t.
\end{align*}
Here, (a) follows from the fact that $\Dphi(\x^{*},\x_{t})\le R^2$. Now using $\eta_t=\frac{R}{\sigma\sqrt{t}}$ and $\mu_t=\frac{1}{L+\frac{1}{\eta_t}}$ we arrive at
\begin{align} \nonumber
\sum_{t=1}^T\big[\E F(\x_t)+\E F(\x_{t+1})-F({\x}^*)\big]\ge& - R^2\left(L+\frac{\sigma}{R}\sqrt{T}\right)-\frac{\sigma R}{2}\sum_{t=1}^T \frac{1}{\sqrt{t}}\\
\overset{(a)}{\ge} &- \left(R^2L+2\sigma R\sqrt{T}\right). \label{non-uniform}
\end{align}
Here, (a) follows from the fact that $\sum_{t=1}^T \frac{1}{\sqrt{t}}\le 2\sqrt{T}$. Thus
\begin{align*}
\sum_{t=1}^T 2\E [ F(\vct{x}_t) ]- \text{OPT}=&\sum_{t=1}^T\big[\E F(\x_t)+\E F(\x_{t+1})-F({\x}^*)\big]+\E [F(\vct{x}_1)]- \E[F(\vct{x}_{T+1})]\\
\ge&- \left(R^2L+2\sigma R\sqrt{T}\right) - \E[F(\vct{x}_{T+1})].
\end{align*}
Dividing both sides by $2T$, we obtain 
$$ \frac{1}{T}\sum_{t=1}^T \E [ F(\vct{x}_t) ]  + \frac{1}{2T}\E[F(\vct{x}_{T+1})] \geq \frac{\rm OPT}{2} - \left(\frac{R^2L}{T}+\frac{2\sigma R}{\sqrt{T}}\right).$$
The proof now follows from the fact that  $E[F(\vct{x}_{T+1})] \leq {\rm OPT}$. Note also that from \eqref{non-uniform} we have
$$ \frac{1}{T}\sum_{t=2}^T \E [ F(\vct{x}_t) ]  + \frac{1}{2T} (\E[F(\vct{x}_{1}) + \E[F(\vct{x}_{T+1})] )\geq \frac{\rm OPT}{2} - \left(\frac{R^2L}{T}+\frac{2\sigma R}{\sqrt{T}}\right).$$
Therefore, a different sampling of the $\vct{x}_t$'s (i.e. choose $\vct{x}_1, \vct{x}_{T+1}$ w.p. $1/(2T)$ and the rest with probability $1/T$--call this sampling $\tau'$) results in   
$\E[F(\vct{x}_{\tau'})] \geq \frac{\rm OPT}{2} - \left(R^2L+2\sigma R\sqrt{T}\right)$.
\subsection{Proof of (stochastic) gradient method (Proof of Theorem~\ref{thm:sgd})}
\label{proofofavg}
This proof is a special case of the proof of the previous section using the mapping $\Phi(\vct{x})=\frac{1}{2}\twonorm{\vct{x}}^2$.

\subsection{Extensions to weakly submodular functions}
\label{weaksub}
In this section we shall show that Theorem~\ref{thm:mirror} extends to weakly submodular functions with the new guarantee given by
\begin{align}
\label{optguaranteefin}
\E[F(\x_\tau)] \ge  \frac{\gamma^2}{1+\gamma^2}\rm{OPT}- \frac{\gamma}{1+\gamma^2}\left(\frac{R^2L+\text{OPT}}{2T}+\frac{R\sigma}{\sqrt{T}}\right).
\end{align}

To this aim using Lemma \ref{key} with $\z=\vct{x}^*$ (Global optimum) and $\eta=\eta_t$ we have
\begin{align*}
\big\langle \Gt, \x_t-\x^{*}\big\rangle\ge& \frac{1}{\mu_t} \left(\Dphi(\x^{*},\x_{t+1})-\Dphi(\x^{*},\x_{t})\right)+F(\x_t)-F(\x_{t+1})-\frac{\eta_t}{2}\|\nabla F(\x_t)-\Gt\|_{*}^2.
\end{align*}
Using $\big\langle \Gt, \x_t-\x^{*}\big\rangle=\big\langle \nabla F(\x_t), \x_t-\x^{*}\big\rangle+\big\langle \Gt-\nabla F(\x_t), \x_t-\x^{*}\big\rangle$ we conclude that
\begin{align*}
\big\langle \nabla F(\x_t), \x_t-\x^{*}\big\rangle\ge& \big\langle \Gt-\nabla F(\x_t), \x_t-\x^{*}\big\rangle\\
&+\frac{1}{\mu_t} \left(\Dphi(\x^{*},\x_{t+1})-\Dphi(\x^{*},\x_{t})\right)+F(\x_t)-F(\x_{t+1})-\frac{\eta_t}{2}\|\nabla F(\x_t)-\Gt\|_{*}^2.
\end{align*}
Using condition \eqref{mahdicond} with $\y=\x^*$ and $\x=\x_t$ in the above inequality we conclude that
\begin{align*}
\left(\gamma+\frac{1}{\gamma}-1\right)F(\x_t)+F(\x_{t+1})-\gamma F({\x}^*)\ge& \big\langle \Gt-\nabla F(\x_t), \x_t-\x^{*}\big\rangle\\
&+\frac{1}{\mu_t} \left(\Dphi(\x^{*},\x_{t+1})-\Dphi(\x^{*},\x_{t})\right)-\frac{\eta_t}{2}\|\nabla F(\x_t)-\Gt\|_{*}^2.
\end{align*}
Taking expectation of both sides we arrive at
\begin{align*}
\left(\gamma+\frac{1}{\gamma}-1\right)\E F(\x_t)+\E F(\x_{t+1})-\gamma F({\x}^*)\ge& \frac{1}{\mu_t} \left(\E\Dphi(\x^{*},\x_{t+1})-\E\Dphi(\x^{*},\x_{t})\right)-\frac{\eta_t}{2}\E\big[\|\nabla F(\x_t)-\Gt\|_{*}^2\big],\\
\ge& \frac{1}{\mu_t} \left(\E\Dphi(\x^{*},\x_{t+1})-\E\Dphi(\x^{*},\x_{t})\right)-\frac{\eta_t}{2}\sigma^2.
\end{align*}
Summing both sides from $t=1$ to $T$ we conclude that
\begin{align*}
\sum_{t=1}^T\bigg[\bigg(\gamma+\frac{1}{\gamma}-1\bigg)\E F(\x_t)+\E F(\x_{t+1})-&\gamma F({\x}^*)\bigg]\\
\ge& \sum_{t=1}^T \frac{1}{\mu_t} \left(\E\Dphi(\x^{*},\x_{t+1})-\E\Dphi(\x^{*},\x_{t})\right)-\frac{\sigma^2}{2}\sum_{t=1}^T\eta_t\\
=&\frac{\E\Dphi(\x^{*},\x_{T+1})}{\mu_T}-\frac{\E\Dphi(\x^{*},\x_{1})}{\mu_1}+\sum_{t=1}^{T-1} \E\Dphi(\x^{*},\x_{t+1})\left(\frac{1}{\mu_t}-\frac{1}{\mu_{t+1}}\right)\\
&-\frac{\sigma^2}{2}\sum_{t=1}^T\eta_t\\
\overset{(a)}{\ge}&-\frac{R^2}{\mu_1}+R^2\sum_{t=1}^{T-1}\left(\frac{1}{\mu_t}-\frac{1}{\mu_{t+1}}\right)-\frac{\sigma^2}{2}\sum_{t=1}^T\eta_t\\
=&-\frac{R^2}{\mu_T}-\frac{\sigma^2}{2}\sum_{t=1}^T\eta_t.
\end{align*}
Here, (a) follows from the fact that $\Dphi(\x^{*},\x_{t})\le R^2$. Now using $\eta_t=\frac{R}{\sigma\sqrt{t}}$ and $\mu_t=\frac{1}{L+\frac{1}{\eta_t}}$ we arrive at
\begin{align*}
\sum_{t=1}^T\bigg[\bigg(\gamma+\frac{1}{\gamma}-1\bigg)\E F(\x_t)+\E F(\x_{t+1})-\gamma F({\x}^*)\bigg]\ge& - R^2\left(L+\frac{\sigma}{R}\sqrt{T}\right)-\frac{\sigma R}{2}\sum_{t=1}^T \frac{1}{\sqrt{t}}\\
\overset{(a)}{\ge} &- \left(R^2L+2\sigma R\sqrt{T}\right).
\end{align*}
Here, (a) follows from the fact that $\sum_{t=1}^T \frac{1}{\sqrt{t}}\le 2\sqrt{T}$. Thus
\begin{align*}
\sum_{t=1}^T \bigg(\gamma+\frac{1}{\gamma}\bigg)\E [ F(\vct{x}_t) ]- \gamma\text{OPT}=&\sum_{t=1}^T\bigg[\bigg(\gamma+\frac{1}{\gamma}-1\bigg)\E F(\x_t)+\E F(\x_{t+1})-\gamma F({\x}^*)\bigg]+F(\vct{x}_1)-F(\vct{x}_{T+1})\\
\ge&- \left(R^2L+2\sigma R\sqrt{T}+\text{OPT}\right).
\end{align*}
Dividing both sides by $\left(\gamma+\frac{1}{\gamma}\right)T$ concludes the proof.

\section*{Acknowledgements}
This work was done while the authors were visiting the Simon's Institute for the Theory of Computing. The authors would like to thank Jeff Bilmes, Volkan Cevher, Maryam Fazel, Mohammad-Reza Karimi, Andreas Krause, Mario Lucic, and Andrea Montanari for helpful discussions.

\bibliography{references}
\bibliographystyle{unsrt} 

\appendix
\section{A DR-Submodular Function that Attains $\rm{OPT}/2 + \epsilon$ on a local maximum}\label{exmponehalf}
We first define a (coverage) submodular set function $f : 2^V \to \mathbb{R}_+$ and then show that the multilinear extension of $f$ has the desired property.  Let $V = \{1,2,\ldots, 2k+1\}$. We consider the following subsets of $V$: For $i \in \{1,2,\ldots,k\}$ let $S_i = \{i, 2k+1\}$, for $i\in\{k+1, \ldots, 2k\}$ define $S_i = \{i\}$, and finally let $S_{2k+1} = \{1,\ldots,k\} \cup \{2k+1\}$. The submodular set function $f: 2^V \to \mathbb{R}_+$ is then defined as $f(A) = | \cup_{i \in A} S_i | $ for $A \subseteq V$. It is not hard to show that $f$ is  monotone and submodular  ($f$ is a coverage function). Let $F: [0,1]^{2k+1} \to  \mathbb{R}_+$ be the multilinear extension of $f$. We can write

\begin{align*}
F(\vct{x}) &= \sum_{A \subseteq V} f(A) \prod_{i \in A} x_i \prod_{i \notin A} (1 - x_i) \\
&= k+1 - (1-x_{2k+1})\prod_{i=1}^k (1-x_i) -  (1-x_{2k+1}) (k - \sum_{i=1}^k x_i)  +  \sum_{i=k+1}^{2k}x_i
\end{align*}    
Now, define $\mathcal{K} = \{\vct{x} \in [0,1]^{2k+1}: \sum_{i=1}^{2k+1} x_i = k\}$. We claim that $\vct{x}_{\rm loc} = (\overbrace{1,1, \ldots, 1}^{k}, 0, 0, \ldots, 0)$ is a local maximum.  To see this, we have 
$\nabla F(\vct{x}_{\rm loc} ) = (\overbrace{1,1\cdots,1}^{2k},0)$. As a result, for any $y \in \mathcal{K}$:
$$\langle \nabla F(\vct{x}_{\rm loc} ), y -  \vct{x}_{\rm loc}  \rangle= \sum_{i=1}^{2k} y_i- k \leq 0. $$
As a result, $ \vct{x}_{\rm loc} $ is a stationary point. It remains to show that in a sufficiently small neighborhood of $\vct{x}_{\rm loc}$ inside $\mathcal{K}$, $\vct{x}_{\rm loc} $ becomes the maximizer of $F$. Note that $F( \vct{x}_{\rm loc} ) = k+1$.  Consider a point 
\begin{equation} \label{local-y}
 y  = (1-\epsilon_1, \cdots, 1 - \epsilon_k, \epsilon_{k+1}, \epsilon_{k+2}, \cdots, \epsilon_{2k+1}), \text{ where }  \epsilon_i \in [0,\epsilon] \text{ and } \sum_{i=1}^k \epsilon_i = \sum_{j=k}^{2k+1} \epsilon_j .
 \end{equation}
It is easy to see that $y \in \mathcal{K}$. We have
\begin{align*}
F(y) - F(\vct{x}_{\rm loc}) &= \sum_{j=k+1}^{2k} \epsilon_j - (1-\epsilon_{2k+1}) \prod_{i=1}^k \epsilon_i - (1-\epsilon_{2k+1}) (k - \sum_{i=1}^k (1-\epsilon_i)) \\
& = \sum_{i=1}^k \epsilon_i - \epsilon_{2k+1} - (1-\epsilon_{2k+1}) ( \prod_{i=1}^k\epsilon_i + \sum_{i=1}^k \epsilon_i) \\
& \leq \epsilon_{2k+1} ( \sum_{i=1}^k \epsilon_i - 1) \\
& \leq \epsilon_{2k+1} ( k \epsilon - 1)  
\end{align*} 
Thus, by choosing $\epsilon \leq k$, we conclude that any $y$ with form as in \eqref{local-y} has a lower function value than $\vct{x}_{\rm loc}$. This proves that  $\vct{x}_{\rm loc}$ is a local maximum. 
Now, consider the vector $\vct{x}^* = (0,0,\cdots,0, \overbrace{1,1,\cdots,1}^{k+1})$. We have $F(\vct{x}_{\rm loc})/F(\vct{x}^*) = 1/2 + 1/(2k)$. As a result, by considering a large enough $k$, the value of $F$ at the  local maximum $\vct{x}_{\rm loc}$ becomes $\rm{OPT}/2 + \epsilon$. 

\section{An Example for Deficiency of the Frank-Wolfe Type Algorithm of \cite{bian2016guaranteed} in the Stochastic Setting} \label{bad}
Assume we want to maximize a DR-Submodular function $F$ over a convex set $\mathcal{K}$. Assume further that $0 \in \mathcal{K}$. The Frank-Wolfe Type algorithm discussed in  \cite{bian2016guaranteed} can be briefly stated as follows (note that for simplicity we let $\alpha=1$ and $\delta = 0$, see Algorithm 1 in \cite{bian2016guaranteed}): Fix a (large) number $T$ as the total number of iterations, let $\vct{x}_0 = 0$ and for $t  < T$ do: 
\begin{equation} \label{andrew}
\vct{x}_{t+1} = x_t + \frac{1}{T} \underset{v \in \mathcal{K}}{\arg\max} \langle v, \nabla F(\vct{x}_t) \rangle.
\end{equation}
When we have access to the gradients $\nabla F (\vct{x}_t)$, it is shown in \cite{bian2016guaranteed} that this algorithm achieves $F(\vct{x}_T) \geq (1-1/e) {\rm OPT}$ for large $T$. Assume now that we have only access to $\vct{g}_t$ which is an unbiased estimator of $\nabla F (\vct{x}_t)$. A simple stochastic version of the above algorithm would be to replace the gradient term $\nabla F (\vct{x}_t)$ in \eqref{andrew}  with $\vct{g}_t$. Here, we provide a simple example to show that this stochastic version can perform arbitrary poorly.  Fix an integer $n$ and consider $n-1$ functions $F_i :[0,1]^n \to \mathbb{R}_+$, $i \in \{1,\cdots,n-1\}$, defined as $F_i(x) = \sum_{j=1}^n m_{i,j} x_j$. We further let $m_{i,i} = 1$, $m_{i,n} = 1/2$, and the rest of $m_{i,j}$'s are $0$. Finally we let $F(x) = \mathbb{E}_{i \sim U}[F_i(x)]$, where $U$ is assumed to be the uniform distribution over $\{1,\cdots,n-1\}$. We want to maximize $F$ over the polytope $\mathcal{K} = \{x: \sum_{i=1}^n x_i \leq 1; x_i \geq 0\}$. 

Assume now that instead of $\nabla F(x)$ we have access to an unbiased estimator $\nabla F_i(x)$ where $i \sim U$. Note that  $\nabla F_i(x) = (m_{i,1}, m_{i,2}, \cdots, m_{i,n})$. As a result, for $i \in \{1,\cdots, n-1\} $we obtain $\arg\max_{v \in \mathcal{K}} \langle \nabla F_i(x), v \rangle = e_i $, where $e_i$ is the vector that has $1$ at position $i$ and $0$ elsewhere. Interestingly for this example, the stochastic Frank-Wolf algorithm never makes any progress on the $n-th$ coordinate and $\vct{x}_t$ will always take $0$ on the $n$-th coordinate. As a result, it is easy to see that for large $T$ the algorithm will end up at $\vct{x}_\infty = (1/(n-1), 1/(n-1), \cdots, 1/(n-1),0)$. However, we have $\vct{x}^* = (0,0,\cdots, 0, 1)$ and $F(\vct{x}_\infty)/F(\vct{x}^*) = 2/(n-1)$ which can become arbitrarily small with $n$.  

Let us briefly explain why conditional gradient methods do not easily admit stochastic variants. The main bottleneck is in the update step of the continuous greedy algorithm (FW). As stated above, in each iteration, FW finds a point in the constraint set $\mathcal{K}$ which has the highest inner product with the gradient and then uses this vector in order to update the current position. However, this step is not very robust to the noise. More precisely, if instead of the gradient of $F$ we plug into the $\arg\max$ a noisy (and unbiased) version of the gradient, the outcome may be far from $\vct{v}_t$. In other words, expectation and $\arg\max$ are not interchangeable. It is easy to see that the above example extends to FW with any fixed natch size (i.e. when gradient is approximated by averaging a fixed number of i.i.d. samples). 

\section{DR-submodular Functions with Large Smoothness Parameter in $\ell_2$ But Reasonable Smoothness Parameter in $\ell_1$}
 \label{example-smooth}
 Consider a submodular set function $f : 2^V \to \mathbb{R}$ and denote its mulilinear extension by $F:[0,1]^n \to \mathbb{R}_+$ (we also let $n \triangleq |V|$). Let us first  investigate how smooth is $F$ under the $\ell_2$ norm. At $\vct{x} = \vct{0}$ we have $[\nabla^2 F(\vct{0})]_{i,j} = f(\{i,j\}) - f(\{i\}) - f(\{j\})$. Thus, for $\vct{y} = (1,1,\cdots,1)/\sqrt{n}$ with $||\vct{y}||_{\ell_2} =1$ we have
$| y^T \nabla^2 F(\vct{0}) y |  = 1/n \sum_{i,j} (f(\{i\}) + f(\{j\}) - f(\{i,j\}) )$. One can easily construct a function $f$ such that this sum takes value $O(n)$ (also many functions in practice  have this property). As a result, $F$ can be $O(n)$-smooth. However, $F$ may become reasonably smooth in $\ell_1$ norm (with smoothness parameter that is independent of the dimension). 
 \begin{lemma}
 For a monotone submodular function $f$, let $m_f$ its denote maximum singleton value of $f$ by $m$, i.e., $m_f \triangleeq \max_{j \in V} f(\{j\})$. 
 Then, the multilinear extension $F$ is $m_f$-smooth under the $\ell_1$ norm. 
 \end{lemma}
 Before proving the lemma, let us remark that in many practical applications, the value of $m_f$ is not so large (see for example the movie recommendation setting of Section~\ref{sec:experiments} where $m_f$ is less than the maximum possible rating).

 \begin{proof}
 At any point $\vct{x} \in[0,1]^n$, the Hessian of $F$, denoted by $\nabla^2 F(x)$, has the following property (see \cite{calinescu11maximizing}):
\begin{align*}
[\nabla^2 F(\vct{x})]_{i,j}& = \frac{\partial^2 F (\vct{x})}{\partial x_i \partial x_j} \\
&=  F(\vct{x}; x_i,x_j \leftarrow 1) + F(\vct{x}; x_i,x_j \leftarrow 0) - F(\vct{x}; x_i \leftarrow 1, ,x_j \leftarrow 0) -   F(\vct{x}; x_i \leftarrow 0, ,x_j \leftarrow 1)  \\
& \stackrel{(a)}{\geq} - \max\{f(\{i\}) , f(\{j\}) \} \geq - m_f,
\end{align*}
where for example by $(\vct{x}; x_i,x_j \leftarrow 1)$ we mean a vector which has value $1$ on its $i$-th and $j$-th coordinate and is equal to $\vct{x}$ elsewhere. Also, (a) is a direct consequence of the submodularity of $f$. As a result, each element of the Hessian is negative but greater $-m$ and for any vector $\vct{y} \in \mathbb{R}^n$ we have $|\vct{y}^T \nabla^2 F(\vct{x}) \vct{y} | \leq m_f || \vct{y} ||_{\ell_1}^2$. Hence, $F$ is $m_f$-smooth under the $\ell_1$ norm. 
\end{proof}

\section{How to Construct an Unbiased Estimator of the Gradient in Multilinear Extensions}\label{unbiased}
Recall that $F(\vct{x}) = \mathbb{E}_{\theta \sim \mathcal{D}} [F_\theta(\vct{x})]$. So $\nabla F_\theta (\vct{x})$ is an unbiased estimator of $\nabla F(\vct{x})$ when $\theta \sim \mathcal{D}$. Note that $F_\theta$ is a multilinear extension. It remains to provide an unbiased estimator for a generic multilinear extension  $G (\vct{x})$. We have $G(\vct{x}) = \sum_{S\subseteq V} \prod_{i\in S} x_i \prod_{j\not\in S} (1-x_j) g(S)$ where $g$ is a set function. Now, it can easily be shown that 
$$\frac{\partial G}{\partial x_i} = G(\vct{x}; x_i \leftarrow 1) - G(\vct{x}; x_i \leftarrow 0).$$
where for example by $(\vct{x}; x_i \leftarrow 1)$ we mean a vector which has value $1$ on its $i$-th coordinate and is equal to $\vct{x}$ elsewhere. To create an unbiased estimator for $\frac{\partial G}{\partial x_i} $ at a point $\vct{x}$ we can simply sample a set $S$ by including each element in it independently with probability $x_i$ and use $g(S \cup \{i\}) - g(S \setminus \{i\})$ as an unbiased estimator for the $i$-th partial derivative. We can sample one single $S$ set and use the above trick for all the coordinates.  This involves $n$ function computations for $g$. Having a batch size $B$ we can repeat this procedure $B$ times and then average.  
\end{document}